\documentclass{article}
\usepackage{iclr2027_conference,times}
\usepackage[T1]{fontenc}
\usepackage[utf8]{inputenc}
\usepackage{textcomp}

\usepackage{amsmath,amssymb,amsthm}
\usepackage{graphicx}
\usepackage{booktabs,array,longtable}
\usepackage{xcolor}
\usepackage{xspace}
\usepackage{url}

\graphicspath{{figures/}}
\newcolumntype{L}[1]{>{\raggedright\arraybackslash}p{#1}}
\newcolumntype{R}[1]{>{\raggedleft\arraybackslash}p{#1}}

\newcommand{\duorate}{\textsc{DuoRate}\xspace}
\newcommand{\bitsmoe}{\textsc{BitsMoE}\xspace}
\newcommand{\bpw}{\ensuremath{\text{bpw}}\xspace}
\newcommand{\Rstore}{\ensuremath{R_{\text{store}}}\xspace}
\newcommand{\Rread}{\ensuremath{R_{\text{read}}}\xspace}

\newcommand{\best}[1]{\textbf{#1}}
\newcommand{\na}{\textcolor{gray}{--}}          %
\newcommand{\yes}{\checkmark}
\newcommand{\no}{\ensuremath{\times}}

\newenvironment{tightenum}%
  {\begin{enumerate}\setlength{\itemsep}{1pt}\setlength{\parskip}{0pt}%
   \setlength{\topsep}{2pt}}%
  {\end{enumerate}}
  {\begin{itemize}\setlength{\itemsep}{1pt}\setlength{\parskip}{0pt}%
   \setlength{\topsep}{2pt}}%
  {\end{itemize}}

\theoremstyle{plain}
\newtheorem{proposition}{Proposition}
\newtheorem{corollary}{Corollary}
\theoremstyle{definition}

\usepackage[hidelinks]{hyperref}
\providecommand{\doi}[1]{\href{https://doi.org/#1}{doi:#1}}

\DeclareUnicodeCharacter{2212}{\ensuremath{-}}
\DeclareUnicodeCharacter{00D7}{\ensuremath{\times}}
\DeclareUnicodeCharacter{2248}{\ensuremath{\approx}}
\DeclareUnicodeCharacter{2264}{\ensuremath{\le}}
\DeclareUnicodeCharacter{2265}{\ensuremath{\ge}}
\DeclareUnicodeCharacter{00B1}{\ensuremath{\pm}}
\DeclareUnicodeCharacter{2192}{\ensuremath{\rightarrow}}
\hypersetup{pdftitle={DuoRate: Two-Rate Per-Expert Precision Allocation for Mixture-of-Experts Serving},pdfauthor={Mengqi Guo, Zhe Wei, Yuan Yuan, Jiunn Bin Lim, Boyi Pan, Michael Bi Mi}}

\iclrfinalcopy
\title{DuoRate: Two-Rate Per-Expert Precision\\ Allocation for Mixture-of-Experts Serving}
\author{{\spaceskip=0.8\fontdimen2\font Mengqi Guo, Zhe Wei, Yuan Yuan, Jiunn Bin Lim, Boyi Pan, Michael Bi Mi} \\
Huawei Technologies Ltd.}
\begin{document}
\maketitle
\lhead{Preprint}
\begin{abstract}
Serving a Mixture-of-Experts (MoE) model pays for expert precision twice: the
resident footprint pays for every expert the router can reach, while each batch-one
decode step pays only for the experts it selects.  These are two different rate functionals over the same per-expert widths,
yet published per-expert MoE allocators constrain only the first.  We call the
difference the \emph{store--read gap}; in the released \bitsmoe checkpoint,
$2.54$ bits per weight (\bpw) of storage become $7.50~\bpw$ of decode traffic.  We
propose \duorate, which stores once and reads per workload.  Offline, it stores one
uniform tier of a prefix-readable weight stream and bakes a distortion tensor.  Online, it profiles a workload's own
prefill without labels and solves a Lagrangian rate--distortion allocation under the
decode functional; routing frequency provably cancels from every per-expert
decision, and the solve costs one dot product per expert and a bisection.  A deployment
therefore keeps a single artifact that prefill reads at the stored rate and decode
reads at the workload's own rate.  On Qwen3-30B-A3B, \duorate reading $2.3~\bpw$
exceeds every 2-bit-class baseline by $6.3$ to $47.0$ points of average accuracy,
and at $2.9~\bpw$ it matches 4-bit formats that read $1.4$--$1.6~\bpw$ more.  The
advantage is largest under long decoding: on GSM8K, \duorate at $2.9~\bpw$
matches or exceeds every 4-bit format and trails BF16 by only $0.25$ points, and at
$2.6~\bpw$ it still matches or exceeds the 4.5-bit NF4 and FP4.  Against controls
in its own format, the allocation alone is worth $0.07$--$0.11~\bpw$ of read rate.
Code will be released upon publication.
\end{abstract}

\section{Introduction}
\label{sec:introduction}

A Mixture-of-Experts (MoE) language model evaluates only a few feed-forward experts
per token, but every expert the router can select must stay resident.  MoE
deployment is therefore memory-bound, and compressing expert weights is an active
area~\citep{mcmoe,bitsmoe}.  On Qwen3-30B-A3B~\citep{qwen3} the $6{,}144$ experts
occupy $54.0$~GiB in BF16, while a batch-one decode step touches only $8$ of the
$128$ experts in each of its $48$ layers.

This asymmetry gives weight precision two distinct costs.  Precision sets the
resident footprint, which prefill also pays once for every expert it reaches.  It
separately sets the traffic of each decode step, which at batch one is paid only for
the selected experts and dominates the cost of multiplying them~\citep{squeezellm}.
In a fixed-precision checkpoint one choice of widths fixes both costs.  A
\emph{prefix-readable} stream decouples them: one
refinement ladder is stored and a request reads only the prefix it needs, so the
read width can be set per expert and per workload without a second
checkpoint~\citep{anyprecision,matquant,stagq}.  We call the difference between what a
checkpoint stores and what a decode step reads the \emph{store--read gap}.

Existing per-expert MoE allocators constrain only the first cost, whether the budget
is per block~\citep{mcmoe}, per layer~\citep{mxmoe,bitsmoe} or global~\citep{gemq},
and routing frequency, where it appears at all, enters the objective, never the constraint.  Serving
systems adapt expert precision to decode-time conditions~\citep{edgemoe,hobbit} but
do not write decode traffic as a rate to be constrained.  The released \bitsmoe 2-bit
checkpoint shows what this omission costs.  Its expert segment stores
$2.5425~\bpw$, reproducing that paper's own $8.58$~GiB, yet a batch-one decode step
reads $7.5008~\bpw$ from the same weights: the FP16 basis shared by the $128$ experts
of a layer is amortized over all of them in storage but over only the $8$ selected at
decode, so its read rate is $2.95\times$ its stored rate.  The field-by-field audit behind
these figures, together with the full evaluation protocol and further analyses, is
provided in the supplementary material.

\begin{figure*}[t]
    \centering
    \includegraphics[width=1.00\linewidth]{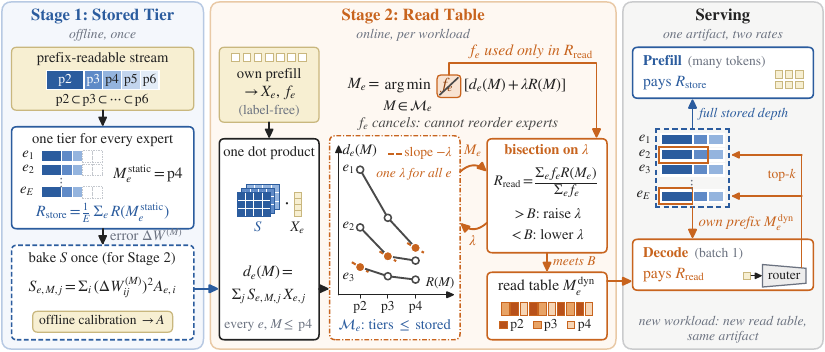}
    \caption{\textbf{The \duorate pipeline.}  \textbf{Stage~1} (offline, once) stores
    one uniform tier of a prefix-readable stream and bakes the distortion tensor $S$.
    \textbf{Stage~2} (per workload) profiles the workload's own prefill without labels,
    scores every expert and tier by one dot product, and takes a per-expert Lagrangian
    $\arg\min$ with one shared $\lambda$, from which $f_e$ cancels; a bisection on
    $\lambda$ meets the read budget $B$.  \textbf{Serving}: prefill pays \Rstore, decode
    pays \Rread.}
    \label{fig:framework}
\end{figure*}

Our starting point is that the two costs call for two decisions made at two
different times: what to \emph{store}, once for every workload, and what to
\emph{read}, separately for each.  Writing the decode cost down as a rate functional
shows that the second decision is cheap.  Routing frequency multiplies an expert's
distortion and its read rate alike, so it cancels from the per-expert
rate--distortion trade-off and cannot reorder experts at any budget.  The algebra is
a known cancellation for mixture sources~\citep{prismquant}; what is specific to MoE
is that the mixture weight is the routing frequency, precisely the statistic prior
allocators treat as importance.  A decode-side allocation therefore needs only the
workload's activation statistics; beyond setting the budget, frequency matters only
through the storage term and, once calibration domains are pooled, through each expert's mixture of
domains.

\duorate turns this into a two-stage pipeline (Figure~\ref{fig:framework}).  In
\textbf{Stage~1}, run once offline, it stores one uniform tier of the stream and
bakes a distortion tensor $S$ from the quantization error of every tier.  In
\textbf{Stage~2}, run per workload, it profiles the workload's own prefill without
labels, scores every expert and tier with one dot product against $S$, and selects
one tier per expert by a Lagrangian $\arg\min$ capped by the stored tier, with a
bisection on the multiplier meeting the read budget.  At \textbf{serving} time a
single resident artifact serves both phases: prefill reads the full stored depth and
pays the stored rate, while decode reads each selected expert's own prefix and pays
the read rate.  A new workload needs a new read table, computed in seconds, but not a
new checkpoint.

On Qwen3-30B-A3B, \duorate reading $2.3~\bpw$ exceeds every 2-bit-class baseline by
$6.3$ to $47.0$ points of average accuracy, and reading $2.9~\bpw$ it matches 4-bit
formats that read $1.4$--$1.6~\bpw$ more.  Its advantage is largest under long
decoding, where the read rate is paid again on every generated token: on GSM8K it
matches or exceeds every 4-bit format at $2.9~\bpw$ and stays within $0.25$ points of
BF16.  Against controls that share its format, the allocation itself is worth
$0.07$--$0.11~\bpw$ of read rate, and every rate-matched comparison with them excludes zero.

\paragraph{Contributions.}
\begin{tightenum}
    \item We formalize MoE weight precision as two rate functionals, a resident rate
    and a routing-weighted decode rate, and show on a released checkpoint that
    accounting for the first alone hides a $2.95\times$ store--read gap.

    \item We prove that routing frequency cancels from decode-side allocation within
    a domain and survives through exactly two channels, so a per-workload read table
    depends on frequency only through its budget.

    \item We propose \duorate, a store-once, read-per-workload allocator whose read
    tables select byte prefixes of one stored tier, so a deployment keeps a single
    artifact with two rates, and whose online stage is label-free: one dot product
    per expert and a bisection.

    \item Under one harness, \duorate outperforms 2-bit-class baselines at
    $2.3~\bpw$, matches 4-bit formats at $2.9~\bpw$, and matches or exceeds them on long
    decoding, while paired same-format controls isolate what the allocation, as
    opposed to the format, contributes.
\end{tightenum}

\section{Related Work}
\label{sec:related}

\paragraph{Quantizing MoE experts.}
General-purpose post-training quantizers such as GPTQ~\citep{gptq} and
HQQ~\citep{hqq} apply one bit-width to every expert; MoE-specific variants add
low-rank compensators (MiLo;~\citealp{milo}) or expert-balanced calibration
(MoEQuant;~\citealp{moequant}).  Mixed-precision methods instead assign different
widths to different experts, typically by an integer program over a distortion
surrogate under one bit budget.  MC-MoE~\citep{mcmoe} and MC\#~\citep{mcsharp}
weight an output-domain error by a frequency-and-gate prior; MxMoE~\citep{mxmoe}
allocates per layer under a memory constraint; GEMQ~\citep{gemq} allocates globally
over the experts of the model we study; \bitsmoe~\citep{bitsmoe} allocates over the
spectral components of a shared SVD basis; MODE~\citep{mode} splits routing
frequency by modality; and MoPEQ~\citep{mopeq} orders experts by a Hessian trace.
Read against their published objectives, none of them constrains a routing-weighted
rate: those that hold a budget hold the unweighted functional, frequency enters, if
at all, the objective rather than the constraint, and each reports a single
bits-per-weight figure; none distinguishes the width a model stores from the width a
decode step reads.

\paragraph{Phase asymmetry in MoE serving.}
MoE serving work establishes that prefill loads each expert once and amortizes it
over the tokens routed to it, whereas decode touches a few experts per
step~\citep{duoserve,moeprefill}; two refinements are known: chunked prefill
reloads experts per chunk~\citep{layeredprefill}, and a batched decode step loads
every distinct expert its batch activates~\citep{eldr}.  Phase-dependent precision
has been proposed in both directions~\citep{pmpd,mixquant}, and expert-offloading
systems set per-expert precision against decode traffic~\citep{edgemoe,hobbit}, in
each case by a sensitivity heuristic or a cache policy rather than under a rate
constraint.  The asymmetry is thus well known on the serving side, while
per-expert quantization inherits a single-rate model from the compression
literature; the two have not been written into one objective.

\paragraph{Rate--distortion allocation and multi-precision formats.}
Lagrangian bisection for rate--distortion allocation~\citep{shohamgersho} has been
applied to dense LLM weights~\citep{radio}, the KV cache~\citep{ratequant} and, with
Hessian curvature, per layer~\citep{hawqv2}.  PrismQuant~\citep{prismquant} proves
that a Gaussian-mixture component's weight multiplies both its rate and its
distortion and cancels at stationarity, and \citet{effros1999} bounds what a
fixed-rate multiresolution code loses against a workload-conditioned one.  The
capability our allocator relies on, reading a prefix of a stored code, comes from
multi-precision formats~\citep{anyprecision,matquant,anybcq,rrq,stagq}, and such codes have
been built into MoE serving~\citep{d2moe,slicemoe} without a rate--distortion objective.

\section{Method}
\label{sec:method}

\paragraph{The weight stream.}
\label{sec:method-stream}
\duorate assigns one precision per expert on StagQ~\citep{stagq}, a prefix-readable
weight stream: each weight has a two-bit group-wise affine base code followed by up to
four one-bit refinement planes on a dyadic step schedule, stored in plane order over
groups of $32$ columns.  Every tier is therefore a byte prefix of the stream, decoded by
one affine map from metadata shared by all tiers, and reading through depth $M$ costs
$M$ payload bits.  Counting every stored field, including the per-group metadata and
a sparse side record that holds out the few weights the grid serves worst, the rate
of tier p$M$ is
\begin{equation}
  R(M)=R(2)+(M-2),\qquad R(2)=2.2704~\bpw,
  \label{eq:instance-rates}
\end{equation}
so tiers p2 through p6 cost $2.2704,\dots,6.2704~\bpw$.  The allocator relies on only
three properties of the format: the tiers form a finite set of known rates, a lower
tier is a byte prefix of a higher one, and changing an expert's tier requires no
re-encoding.

\subsection{Two rate functionals, and where routing frequency survives}
\label{sec:method-rates}

Let expert $e$ sit at precision $M_e$.  Prefill processes $T$ tokens together and
fetches an expert's weights if \emph{any} of them routes to it, reusing the weights
across all of them~\citep{duoserve,moeprefill}; with top-$k$ routing a prompt of a few
hundred tokens reaches nearly every expert, so prefill traffic is the same
unweighted sum as the resident footprint.  Batch-one decode is the opposite limit
$T{=}1$: nothing amortizes the fetch, and the expected traffic is weighted by how
often the router selects each expert.  Writing $f_e$ for that decode routing count,
\begin{equation}
  \Rstore=\frac{1}{E}\sum_{e=1}^{E}R(M_e),
  \qquad
  \Rread=\frac{\sum_{e}f_{e}\,R(M_e)}{\sum_{e}f_{e}} .
  \label{eq:two-rates}
\end{equation}
Under a uniform precision with no shared components the two coincide.  A basis
shared across the experts of a layer makes \Rread{} exceed \Rstore, as for \bitsmoe; a
prefix-readable stream allows $\Rread<\Rstore$ (Figure~\ref{fig:main}).  At
batch size $b>1$, $f_e$ is replaced by $1-(1-p_e)^{b}$, with $p_e$ the per-token
probability of selecting $e$, so \Rread{} returns to the unweighted form as the batch grows.
The two functionals differ not only in value but in which statistics they admit
into the assignment.

\begin{proposition}[Frequency cancellation]
\label{prop:cancellation}
Let $\mathcal M$ be a finite precision set with rate function
$R:\mathcal M\to\mathbb R$, let $d_e:\mathcal M\to\mathbb R_{\ge0}$ be per-expert
distortion curves and let $f_e>0$.  Consider the allocation problem in which the
same weights appear in the objective and in the constraint,
\[
  \min_{M_1,\dots,M_E}\ \sum\nolimits_e f_e\,d_e(M_e)
  \quad\text{s.t.}\quad
  \sum\nolimits_e f_e R(M_e)\ \le\ B\sum\nolimits_e f_e .
\]
For every $\lambda\ge0$ the per-expert minimizer of the Lagrangian relaxation is
$M_e(\lambda)=\arg\min_{M\in\mathcal M}\{d_e(M)+\lambda R(M)\}$, which does not
depend on $f$.  The family swept by the multiplier is identical for every positive
weight vector, and $f$ only selects which member of it meets the budget.
\end{proposition}

\begin{corollary}
\label{cor:no-reorder}
Routing frequency cannot reorder two experts: whether $e$ receives a deeper prefix
than $e'$ is decided at every $\lambda$ by their distortion curves alone, through the
exchange rate $\Delta d/\Delta R$ of each step.
\end{corollary}

\begin{figure}[t]
\centering
\includegraphics[width=\linewidth]{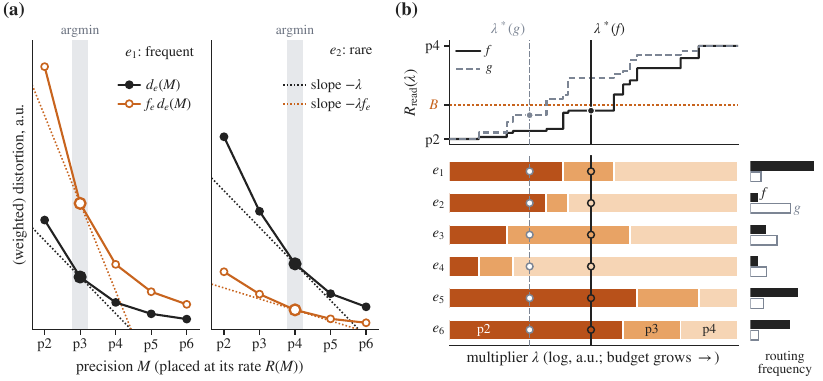}
\caption{Frequency cancellation (Proposition~\ref{prop:cancellation},
Corollary~\ref{cor:no-reorder}); illustrative values.  \textbf{(a)}~For two experts
eightfold apart in $f$, supporting lines of slope $-\lambda$ to $d_e$ and $-\lambda f_e$
to $f_e d_e$ touch the same tier, so the $\arg\min$ does not depend on $f$.
\textbf{(b)}~On the read set $\{\mathrm{p2},\mathrm{p3},\mathrm{p4}\}$ the thresholds fix
one family of tier maps; routing frequencies $f$ and $g$ only change where the bisection stops.}
\label{fig:cancellation}
\end{figure}

\emph{Proof of Proposition~\ref{prop:cancellation}.}  The Lagrangian
$\sum_e f_e[d_e(M_e)+\lambda R(M_e)]-\lambda B\sum_e f_e$ separates across experts, and
in the $e$-th summand both terms carry the same positive factor $f_e$, which does not
move the $\arg\min$.  The budget condition does depend on $f$, so $f$ selects the
member of the family that meets it.  \hfill$\square$

Figure~\ref{fig:cancellation} illustrates the result.  The algebra is not new:
\citet{prismquant} prove the same cancellation for a Gaussian-mixture source.  What is
specific to MoE is that the mixture weight \emph{is} the routing frequency, the very
quantity that prior allocators treat as importance and that, in the decode functional,
carries no information about the assignment.  Frequency can influence the assignment
only where it does not scale an expert's distortion and rate alike, which happens in
exactly two ways.

\paragraph{The cross-domain exception.}
Pooling calibration domains breaks the cancellation in one specific way.  Let domain
$d$ contribute distortion $d^{(d)}_e$ and frequency $f_{d,e}$, each normalized within
the domain, and let the pooled objective weight domains by $w_d$.  With
$F_e=\sum_d w_d f_{d,e}$,
\begin{equation}
  \sum\nolimits_d w_d f_{d,e}\,d^{(d)}_e(M)+\lambda F_e R(M)
  = F_e\Bigl[\sum\nolimits_d \pi_{d\mid e}\,d^{(d)}_e(M)+\lambda R(M)\Bigr],
  \qquad
  \pi_{d\mid e}=\frac{w_d f_{d,e}}{F_e}.
  \label{eq:multidomain}
\end{equation}
The total frequency $F_e$ cancels as before, but its \emph{distribution} over
domains does not, because it weights the curves being averaged.

\begin{proposition}[The two channels]
\label{prop:channels}
Routing frequency enters a per-expert assignment in exactly two ways: through a
bit-width term whose weight is not proportional to the distortion weight, which under
an \Rstore{} budget with multiplier $\lambda_{\mathrm s}$ gives the marginal cost $\lambda_{\mathrm s}/(Ef_e)$ and makes
a frequently routed expert cheaper to deepen; and, once calibration domains are
pooled, through the per-expert mixture $\pi_{d\mid e}$ of
Equation~\eqref{eq:multidomain}.  In no other role does it survive.
\end{proposition}

This is the structural reason for two stages.  A table read by a single workload is
solved on that workload's statistics alone, where frequency has nothing left to do
but set the budget; a table shared by all workloads, which is what is stored, is the
only place where pooling, and with it the mixture channel, can enter.

\subsection{\duorate: store once, read per workload}
\label{sec:method-allocator}

\duorate assigns one precision per expert, shared by its three matrices, giving
$E=6{,}144$ allocation units; splitting the matrices into separate units did not
measurably help in our earlier experiments.  The pipeline has an offline stage, an
online stage and a serving path (Figure~\ref{fig:framework}).

\paragraph{Stage 1: store one tier and bake $S$ (offline, once).}
Stage one fixes what the deployment keeps resident and is constrained on \Rstore{}
alone.  We store the uniform p4 tier, $M^{\mathrm{static}}_e=\mathrm{p4}$ for every
expert and $\Rstore=4.2704~\bpw$, so every read tier is a byte prefix of one resident
artifact and prefill never runs below four bits.  The same stage prepares everything
the online stage needs.  The distortion of expert $e$ at precision $M$ is its squared
weight perturbation, weighted on the input side by the energy of the activations it
sees and on the output side by how much each output channel matters to the model's
decision:
\begin{equation}
  d_e(M)=\sum\nolimits_{ij}\bigl(\Delta W^{(M)}_{ij}\bigr)^2 A_{e,i}\,X_{e,j},
  \qquad X_{e,j}=\mathbb E_t\bigl[x_{t,j}^2\bigr],
  \label{eq:distortion}
\end{equation}
where $\Delta W^{(M)}$ is the error of the stored byte prefix at depth $M$, and
$A_{e,i}\equiv1$ recovers the activation-weighted squared error that prior per-expert
allocators optimize.  Since neither $\Delta W^{(M)}$ nor $A$ depends on the served workload, the
workload-independent part is baked once,
\begin{equation}
  S_{e,M,j}=\sum\nolimits_i\bigl(\Delta W^{(M)}_{ij}\bigr)^2 A_{e,i}
  \quad\Longrightarrow\quad
  d_e(M)=\sum\nolimits_j S_{e,M,j}\,X_{e,j}.
  \label{eq:factorization}
\end{equation}
The output weight $A$ is decision-aligned rather than constant: quantized models that
match a reference on average still flip individual answers where the top-two margin
is small~\citep{flips}, so we weight each output channel by the sensitivity of the
expected number of flipped decisions, obtained from one offline backward pass and
pooled over four reasoning and multiple-choice domains.  An allocated stored table, solved on pooled
prefill statistics, is equally admissible here and is where the mixture channel of
Equation~\eqref{eq:multidomain} would act; we leave it to
future work.

\paragraph{Stage 2: a label-free read table per workload (online).}
For a served workload, \duorate runs the model over the workload's own prefill once,
without labels, and records the per-channel second moments $X_e$ and routing counts
$f_e$, the prefill estimate of the decode counts in Equation~\eqref{eq:two-rates}.  Equation~\eqref{eq:factorization} then gives the full distortion table as one
dot product per expert and tier, and the read table is the per-expert minimizer
\begin{equation}
  M^{\mathrm{dyn}}_e=\arg\min_{M\in\mathcal M_e}\ \bigl[\,d_e(M)+\lambda\,R(M)\,\bigr],
  \qquad \mathcal M_e=\{M: M\le M^{\mathrm{static}}_e\}=\{\mathrm{p2},\mathrm{p3},\mathrm{p4}\},
  \label{eq:argmin}
\end{equation}
with a single multiplier $\lambda$ shared by all experts.  By
Proposition~\ref{prop:cancellation}, $f_e$ does not enter this $\arg\min$; its only
role is to convert $\lambda$ into a read rate, and a bisection raises or
lowers $\lambda$ until \Rread{} meets the budget $B$.  Ties break toward the lower rate, so an
expert the prompts never reach falls to the lowest tier and costs nothing in the read
functional.  No integer program is needed: wherever the per-expert tier family is
convex, the Lagrangian sweep is globally optimal at every rate it reaches.  The cap imposed by the
stored tier rarely binds: without it, only $12$ to $29$ of the $6{,}144$ experts would
rise above p4 on the workloads we checked.

\paragraph{Serving: one artifact, two rates.}
At serving time the resident artifact is read at two depths.  Prefill processes many
tokens, reaches nearly every expert, and reads the full stored depth, paying
\Rstore.  Each decode step routes its token to $k$ experts and reads only their
prefixes $M^{\mathrm{dyn}}_e$, paying \Rread.  Changing workloads replaces the read
table, a few kilobytes of tier indices, and leaves the stored weights untouched.
Because Equation~\eqref{eq:factorization} needs only $X$ from the live stream, the read
table can also be recalibrated online, with no labels and no reference model, when the
served domain shifts.

\begin{table*}[t]
    \centering
    \small
    \caption{Results on Qwen3-30B-A3B-Base.  \textbf{BPW} is the rate read per
    activated expert weight at batch-one decode (for \duorate, the read budget);
    \duorate stores $4.3$.
    $^{*}$\bitsmoe stores $2.5$; $7.5$ is its read rate.  GSM8K is strict match;
    \textsc{Avg.} is the mean of MMLU-Pro, MMLU and GSM8K.  Bold is the best in each
    group.}
    \label{tab:main}
    \setlength{\tabcolsep}{7pt}
    \begin{tabular}{l|c|c|ccc|c}
        \toprule
        \textbf{Method} & \textbf{BPW} & \textbf{PPL}$\downarrow$
        & \multicolumn{3}{c|}{\textbf{Accuracy}$\uparrow$ (\%)} & \textbf{Avg.}$\uparrow$ \\
        \cmidrule(lr){4-6}
        & Read & WikiText-2 & MMLU-Pro & MMLU & GSM8K & \\
        \midrule[1pt]
        BF16      & 16.0 & 8.27 & 56.86 & 83.30 & 91.25 & 77.14 \\
        \midrule[1pt]
        NVFP4     & 4.5 & 8.44 & 56.79 & \best{83.09} & 90.25 & 76.71 \\
        NF4       & 4.5 & 8.41 & 56.79 & 82.11 & 89.50 & 76.13 \\
        FP4       & 4.5 & 8.62 & 56.43 & 81.75 & 88.25 & 75.48 \\
        MXFP4     & 4.3 & 8.58 & 56.71 & 82.53 & 90.25 & 76.50 \\
        HIF4      & 4.5 & \best{8.40} & \best{56.86} & 82.60 & \best{91.00} & \best{76.82} \\
        \midrule[1pt]
        HQQ       & 2.3 & 21.84 & 37.64 & 53.05 & 23.25 & 37.98 \\
        GPTQ      & 2.3 & 14.08 & 30.93 & 43.93 & 5.00  & 26.62 \\
        MiLo      & 2.3 & 18.93 & 35.57 & 48.21 & 21.00 & 34.93 \\
        MoEQuant  & 2.5 & 13.81 & 44.93 & 75.37 & 52.75 & 57.68 \\
        \bitsmoe  & 7.5$^{*}$ & 14.69 & 46.21 & 77.75 & 78.00 & 67.32 \\
        \cmidrule(lr){1-7}
        \duorate  & 2.9 & \best{9.07} & \best{56.00} & \best{81.47} & \best{91.00} & \best{76.16} \\
        \duorate  & 2.6 & 9.46  & 53.86 & 80.56 & 89.50 & 74.64 \\
        \duorate  & 2.3 & 9.97  & 52.64 & 80.00 & 88.25 & 73.63 \\
        \bottomrule
    \end{tabular}
\end{table*}

\section{Experiments}
\label{sec:experiments}

\paragraph{Setup.}
\label{sec:exp-setup}
We quantize Qwen3-30B-A3B-Base~\citep{qwen3}, whose $48$ MoE layers of $128$ experts
give $E=6{,}144$ allocation units.  Every method quantizes only the expert matrices;
attention, routers, embeddings and normalization stay in BF16.  All numbers come from
one harness, lm-eval~\citep{lmeval}, at sequence length $2{,}048$ and batch size one:
word perplexity on the full WikiText-2~\citep{wikitext} test split ($62$ documents),
and accuracy on a stratified MMLU-Pro sample ($100$ questions for each of $14$
subjects, 5-shot), MMLU~\citep{mmlu} ($25$ per subject, 5-shot) and
GSM8K~\citep{gsm8k} (5-shot, strict match).  GSM8K is the one benchmark that decodes
long chains of reasoning; the other three are scored by log-likelihood.  \duorate
stores the p4 tier ($4.2704~\bpw$) and solves one read table over $\{$p2, p3, p4$\}$
per benchmark from that benchmark's own prefill, at read budgets of $2.3$, $2.6$ and
$2.9~\bpw$; calibration items are disjoint from the scored ones on every benchmark but
WikiText-2, which Section~\ref{sec:exp-where} treats separately.  Each budget is met to within $0.001~\bpw$ by a CPU bisection that takes
seconds.

\paragraph{Baselines.}
We compare against five 4-bit formats applied to the same matrices --- NVFP4 (E4M3
block-16 scales plus a per-tensor FP32 scale), NF4 and FP4 as in bitsandbytes (block
$64$), MXFP4 (E2M1 with an E8M0 scale per $32$) and HIF4 (HiFloat4;~\citealp{hif4}), a
4-bit block floating-point format with $64$-element groups and a three-level scale, at
$4.5~\bpw$ --- and against four 2-bit post-training methods,
HQQ~\citep{hqq}, GPTQ~\citep{gptq}, MiLo~\citep{milo} and MoEQuant~\citep{moequant}.
For \bitsmoe~\citep{bitsmoe} we dequantize the released 2-bit checkpoint and score it
with no second quantization pass, after recomputing its stored and read rates from
the released tensors and checking that our harness reproduces its published numbers
where the two protocols coincide.  To isolate the allocation from the format, we
also compare against controls that share \duorate's format and differ only in how
tiers are assigned: a random mixture of two adjacent tiers matched to the same read
rate, and an interpolated null between the uniform tiers.  These comparisons are
paired, using an exact document sign test with a bootstrap interval on the aggregate
perplexity.

\subsection{Main Results}
\label{sec:exp-main}

\begin{figure*}[t]
    \centering
    \includegraphics[width=1.00\linewidth]{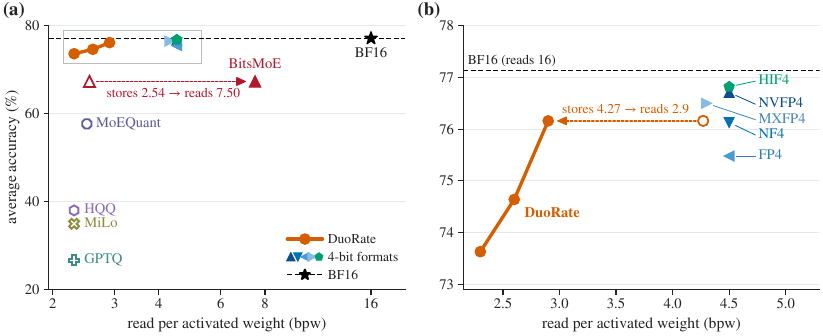}
    \caption{Average accuracy (MMLU-Pro, MMLU, GSM8K) against the rate read per
    activated weight at batch-one decode.  \textbf{(a)}~Every row of
    Table~\ref{tab:main}.  \textbf{(b)}~The top-left corner: \duorate at $2.9~\bpw$ sits
    with the 4-bit formats, which read $1.4$--$1.6~\bpw$ more.  Hollow markers are stored
    rates joined to read rates: \bitsmoe stores $2.54$ and reads $7.50$; \duorate stores
    $4.27$ and reads $2.9$.}
    \label{fig:main}
\end{figure*}

\paragraph{Far above 2-bit-class methods at $2.3~\bpw$.}
At $2.3$--$2.5~\bpw$ the four post-training methods lose $12$--$26$ MMLU-Pro points
and $38.50$--$86.25$ GSM8K points relative to BF16 (Table~\ref{tab:main},
Figure~\ref{fig:main}(a)), whereas \duorate at $2.3~\bpw$ loses $4.22$ and $3.00$.
Its average accuracy of $73.63$ exceeds that of the strongest of them, MoEQuant at
$2.5~\bpw$, by $15.95$ points, with $35.5$ more on GSM8K and a perplexity of $9.97$
against $13.81$.  It also beats \bitsmoe on every benchmark --- by $6.43$ MMLU-Pro,
$2.25$ MMLU and $10.25$ GSM8K points and by $32\%$ in perplexity --- while reading
$3.3\times$ less per decoded token.  Most of this margin comes from the weight format
rather than from allocation; Section~\ref{sec:exp-where} isolates the allocation with
controls in the same format.

\paragraph{On par with 4-bit formats at $2.9~\bpw$.}
Reading $2.9~\bpw$, \duorate stays within $0.86$ points of BF16 on MMLU-Pro and
$1.83$ on MMLU.  Its average of $76.16$ matches NF4, exceeds FP4 and is within $0.66$ points
of the best 4-bit format, HIF4, although every 4-bit format reads $1.4$--$1.6~\bpw$
more per activated weight (Figure~\ref{fig:main}(b)).  The formats keep an edge in
perplexity ($8.40$--$8.62$ against $9.07$) and store a comparable amount, $4.25$ to
$4.5~\bpw$ against \duorate's $4.27$; what \duorate saves is decode traffic.

\paragraph{Strongest under long decoding.}
The advantage is largest on GSM8K, where every generated token pays the read rate
again.  At $2.9~\bpw$ \duorate scores $91.00$, matching the best 4-bit format, HIF4 at
$4.5~\bpw$, is at or above the other four ($88.25$--$90.25$), and is within $0.25$
points of BF16.
At $2.6~\bpw$ it scores $89.50$, still matching or exceeding the $4.5~\bpw$ NF4 and FP4.  Two
properties of the read table suit this regime.  Its budget is charged only for the
experts a workload's tokens select, so no bits are spent on experts that are resident
but idle for that workload, and it is calibrated on held-out prompts of the same workload, so the
experts that carry a reasoning trace are profiled on the kind of tokens they will
process.  The decode traffic actually realized on GSM8K is lower still than the
budget: $2.82~\bpw$ at the $2.9$ budget and $2.56$ at $2.6$.  By contrast, \bitsmoe
reads $7.50~\bpw$ per decoded token while storing $2.54$ (Figure~\ref{fig:main}(a)).

\subsection{What the Allocation Is Worth}
\label{sec:exp-where}

\begin{figure*}[t]
    \centering
    \includegraphics[width=1.00\linewidth]{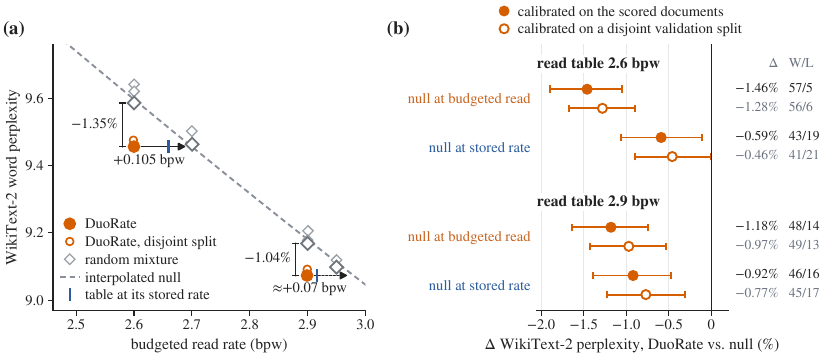}
    \caption{\textbf{What the allocation alone is worth} (WikiText-2).  \textbf{(a)}~Perplexity
    against the budgeted read rate: \duorate, random mixtures of two adjacent tiers (every
    draw), the interpolated uniform-tier null (dashed) and tables re-solved from a disjoint
    split (hollow).  Brackets: gap to the strongest mixture; arrows: extra read rate the
    mixture needs to match (dotted: extrapolated); blue ticks: each table at its stored
    rate.  \textbf{(b)}~Gap to the null placed at either rate, with $95\%$ bootstrap
    intervals.}
    \label{fig:worth}
\end{figure*}

Figure~\ref{fig:worth} prices the allocation against controls that keep the format and
discard the importance information.  Against the strongest random mixture of two adjacent tiers at
the same read rate, the allocated table wins by $1.35\%$ of perplexity at $2.6~\bpw$
($48$ of $56$ untied documents, $p=4.7\times10^{-8}$) and by $1.04\%$ at $2.9$ ($45$ of
$54$, $p=7.3\times10^{-7}$).  Expressed as rate, the
mixture must read $0.105~\bpw$ more to match the allocated table at $2.6$, and about
$0.07$ more at $2.9$, where the crossing is extrapolated.  Against a second mixture that
is dearer under both rate definitions, the comparisons are ties, so about a tenth of a
bit is the size of the effect rather than a lower bound on it.  The interpolated
uniform-tier null, which can be placed at either rate definition, agrees
(Figure~\ref{fig:worth}(b)): the allocation beats it by $1.46\%$ and $1.18\%$ under the
budgeted read rate and by $0.59\%$ and $0.92\%$ under the table's own stored rate, all four
intervals excluding zero.  Because both controls share every property of the format,
these margins isolate the value of the importance information itself.  The WikiText-2 tables are calibrated on the scored test documents; re-solving them
from the disjoint validation split
changes perplexity by only $0.20\%$, and all four comparisons with the null still
exclude zero, so the gain does not come from calibrating on the scored documents.

\subsection{Choosing the Allocation Objective}
\label{sec:exp-ablation}

The two rate functionals also decide between allocation objectives.  We substitute
MC-MoE's objective~\citep{mcmoe}, which weights the distortion by routing frequency and
leaves the bit term unweighted, into the same bisection with the same distortion table
and cap.  At matched read rates of $2.6$ and $2.9~\bpw$ it loses to ours by $0.42\%$ and $0.61\%$ of perplexity,
a difference the sign test detects but the interval does not, while costing $0.29$ and
$0.43~\bpw$ less under \Rstore: weighting by frequency drives rarely routed experts to
the lowest tier, which saves stored rate but almost no read rate.  Neither
table dominates the other.  Which one is preferable depends on which rate the
deployment holds fixed, a choice only a two-rate formulation can make explicit.

\subsection{Efficiency Accounting}
\label{sec:exp-efficiency}

\begin{table}[t]
    \centering
    \small
    \caption{Resident footprint and batch-one decode traffic of the expert weights,
    implied by each artifact's stored and read rates.  $^{\ddagger}$Read budget; the
    realized decode traffic on GSM8K is lower.}
    \label{tab:efficiency}
    \begin{tabular}{l|cc|cc}
        \toprule
        \textbf{Method}
        & \multicolumn{2}{c|}{\textbf{Resident}}
        & \multicolumn{2}{c}{\textbf{Decode, batch 1}} \\
        \cmidrule(lr){2-3}\cmidrule(lr){4-5}
        & Store (\bpw) & GiB$\downarrow$ & Read (\bpw) & MB/token$\downarrow$ \\
        \midrule
        BF16             & 16.0000 & 54.000 & 16.0000 & 3623.9 \\
        \bitsmoe 2-bit   & 2.5425  & \best{8.581} & 7.5008  & 1698.9 \\
        \midrule
        \duorate $2.3$   & 4.2704  & 14.413 & 2.3000$^{\ddagger}$ & \best{520.9}$^{\ddagger}$ \\
        \duorate $2.6$   & 4.2704  & 14.413 & 2.6000$^{\ddagger}$ & 588.9$^{\ddagger}$ \\
        \duorate $2.9$   & 4.2704  & 14.413 & 2.9000$^{\ddagger}$ & 656.8$^{\ddagger}$ \\
        \bottomrule
    \end{tabular}
\end{table}

The two rates rank the artifacts differently: \bitsmoe is cheaper to hold, $8.58$~GiB
of expert weights against $14.41$ for \duorate, and dearer to decode from, $1{,}699$~MB
per token against $521$--$657$ (Table~\ref{tab:efficiency}).  An allocator that
minimizes one number cannot see this trade.  These figures are computed from the
storage layout; measuring them end to end requires a prefix-read kernel, which we leave
to future work.

\subsection{When the Read Rate Matters}
\label{sec:exp-when}

Reading less than one stores pays most in the regime this paper targets, batch-one
decoding of long outputs, where every generated token pays the read rate again.  The
benefit narrows in two directions.  As the batch grows, the probability that a decode
step touches expert $e$, $1-(1-p_e)^{b}$, saturates, and the routing-weighted rate of
Equation~\eqref{eq:two-rates} approaches the unweighted one, so a large-batch
deployment should read close to what it stores.  For short generations, prefill,
which reads the full stored depth, dominates the traffic.  Consistent with this,
\duorate's standing against formats of comparable storage
is strongest on GSM8K, the one benchmark with long decoding.  Neither limit calls for
a different artifact: a deployment moves along this range by changing the read
budget, without re-encoding any weight.  Changing the budget is also cheap in
practice: profiling a workload takes one forward pass over its prefill, solving a read
table takes seconds on a CPU, and the table itself is a few kilobytes of tier
indices.

\section{Limitations}
\label{sec:limitations}

Our evaluation covers one MoE model, Qwen3-30B-A3B; extending it to other routing
schemes and expert counts is natural future work.  The efficiency figures are derived from the
storage layout rather than measured with a fused prefix-read kernel.  In our
measurements prefill also runs at the read depth rather than from the stored tier,
so the reported quality is a conservative estimate of the two-stage deployment.
Finally, read tables are calibrated per workload; recalibration is cheap, one dot
product per expert, but a deployment whose request mix drifts must schedule it.

\section{Conclusion}
\label{sec:conclusion}

An MoE serving path charges weight precision twice: once as the resident capacity
that prefill also pays, and once as the traffic of every batch-one decode step.  A
prefix-readable stream makes the two independently settable, and a released 2-bit
checkpoint shows how far they can drift apart when only one is constrained: the same
bytes that cut its resident footprint by $6.29\times$ cut decode traffic by only
$2.13\times$.  Writing both rates down also shows that routing frequency cannot order
experts in the decode objective, so the read side can be allocated per workload,
without labels, by one dot product per expert and a bisection.  \duorate builds on
this: it stores one tier once and reads a per-workload prefix of it, so a deployment
keeps a single artifact with two rates.  On Qwen3-30B-A3B it outperforms 2-bit-class
methods at $2.3~\bpw$, matches 4-bit formats at $2.9~\bpw$, and on long decoding
matches or exceeds them while staying within a quarter point of BF16.  We advocate
that efficiency claims for MoE checkpoints report both rates, because a checkpoint
that is cheap to store need not be cheap to decode.  The natural next step is to
measure the read-side savings end to end with a fused prefix-read kernel.

\subsubsection*{Acknowledgments}
We used generative AI tools to help develop the conceptual framework.  We also
used them to write the code, to interpret results, to draft and edit the text and
to produce the figures.  The authors take full responsibility for the content of
this paper.  We thank Twisha Piyush Mehta for the effort and contributions during an
internship.

\bibliographystyle{iclr2027_conference}
\bibliography{refs}

\clearpage
\section*{Supplementary Material}
\appendix
\renewcommand{\thefigure}{S\arabic{figure}}
\renewcommand{\thetable}{S\arabic{table}}
\renewcommand{\theequation}{S\arabic{equation}}
\renewcommand{\theHfigure}{S\arabic{figure}}
\renewcommand{\theHtable}{S\arabic{table}}
\renewcommand{\theHequation}{S\arabic{equation}}
\setcounter{figure}{0}
\setcounter{table}{0}
\setcounter{equation}{0}
\section{Positioning \duorate Among Per-Expert MoE Allocators}
\label{app:positioning}

\begin{table}[h]
\centering
\caption{Objective-function audit of per-expert MoE allocators, read from the
published objectives~\citep{mcmoe,mcsharp,mxmoe,gemq,bitsmoe,mode,alphaq,mopeq,dynaexq,edgemoe,hobbit}.  Rows are restricted to methods whose objectives we
verified in full.}
\label{tab:audit}
\footnotesize
\setlength{\tabcolsep}{4pt}
\resizebox{\linewidth}{!}{%
\begin{tabular}{L{1.6cm}L{1.7cm}L{2.3cm}L{2.6cm}L{1.9cm}cc}
\toprule
\textbf{Method} & \textbf{Unit} & \textbf{Rate constraint} & \textbf{Distortion signal}
& \textbf{Frequency enters} & \textbf{Store $\ne$ read} & \textbf{Paired} \\
\midrule
MC-MoE / MC\# & expert & unweighted, per block
& output error $\times$ freq.$^{\alpha}$gate$^{\beta}$ & distortion prior & \no & \no \\
MxMoE & expert $\times$ proj.\ pair & memory, per layer
& block-output error & cost term only & \no & \no \\
GEMQ & expert, global & unweighted, global
& leave-one-out, CE-gradient & implicit at most & \no & \no \\
\bitsmoe & spectral component & unweighted, per layer
& spectral energy $\times$ analytic $D(b)$ & importance weight & \no & \no \\
MODE & expert, per layer & unweighted, per layer
& modality freq.\ $\times$ KL & distortion weight & \no & \no \\
AlphaQ & projection matrix & unweighted, global
& analytic weight-domain & none & \no & \no \\
MoPEQ & expert & none
& Hessian trace & deliberately none & \no & \no \\
DynaExq & expert, per layer & HBM envelope
& none & is the rule & separate copies & \no \\
EdgeMoE & expert & none
& per-expert accuracy drop & none & \no & \no \\
HOBBIT & expert, per token & none
& criticality heuristic & selects the copy & by cache policy & \no \\
\midrule
\textbf{\duorate} & expert & \textbf{two}: unweighted \emph{and} frequency-weighted
& measured $A\times X$, margin-weighted & read budget only
& \textbf{\yes} & \textbf{\yes} \\
\bottomrule
\end{tabular}%
}

\vspace{2pt}
\begin{minipage}{0.98\linewidth}
\scriptsize
\emph{Note.} ``Rate constraint'' is the inequality the published solver holds;
``Frequency enters'' is where routing statistics appear in the objective;
``Store $\ne$ read'' is whether the resident width and the batch-one decode width
are distinguished; ``Paired'' is whether any accuracy comparison uses a paired
significance test.
\end{minipage}
\end{table}

Three things are visible in Table~\ref{tab:audit}.  First, no published allocator
writes the stored and the read rate as two functionals.  The systems that do vary the
width a decode step reads do so through a residency policy, an accuracy budget or a
cache rather than through a rate: serving systems keep more than one precision
resident~\citep{dynaexq,pagedweight}, EdgeMoE~\citep{edgemoe} shrinks each expert
to a per-expert accuracy budget, and HOBBIT~\citep{hobbit} substitutes a
lower-precision copy when an expert misses the cache.  Each therefore pays the
difference between the two rates without writing either one down.  Second, no published row
puts frequency in the rate constraint, so formulations
that describe themselves as frequency-aware are using the unweighted capacity
constraint as the channel through which routing frequency acts --- a route
Proposition 2 in the main paper makes explicit and none of them names.  Third,
MoPEQ's decision to abandon frequency is empirical; Proposition 1 in the main
paper supplies the condition under which it is right.

Two further facts about these allocators come from their own released artifacts.
MC-MoE's nine released allocations hold the total bit count of every block to a
single value, so its solver has no cross-layer freedom by construction.  GEMQ's
three released Qwen3-30B-A3B allocations carry almost no routing-frequency signal
(Spearman $0.07$--$0.21$ between assigned width and decode frequency), which is
why its stored and read rates nearly coincide --- not because it constrained the
read axis, but because its distortion weighting does not scale with routing
count.

\paragraph{Why there is no Hessian-trace arm.}
MoPEQ constrains no rate at all: it clusters a scalar sensitivity into tiers.  Any
rate-matched instantiation of it would therefore be our own rate-constrained
construction wearing a different importance score, and would measure our allocator
rather than theirs.

\section{Frequency Cancellation: Algebra}
\label{app:cancellation}

\paragraph{Proof of Proposition 1 in the main paper.}
For $\lambda\ge0$ the Lagrangian relaxation of the stated problem is
$L(M,\lambda)=\sum_{e} f_e[d_e(M_e)+\lambda R(M_e)]-\lambda B\sum_{e} f_e$, whose
second term does not involve the decision variables.  The first term is a sum in
which the $e$-th summand depends only on $M_e$, so the minimization separates
across experts.  For each $e$ the factor $f_e$ is positive and common to both
terms of the summand, and scaling an objective by a positive constant does not
move its $\arg\min$; hence
$\arg\min_M f_e[d_e(M)+\lambda R(M)]=\arg\min_M[d_e(M)+\lambda R(M)]$, which
involves no component of $f$.  Sweeping $\lambda$ therefore traces the same
one-parameter family of allocation vectors for every positive weight vector.  The
budget condition $\sum_e f_e R(M_e(\lambda))\le B\sum_e f_e$ does depend on $f$,
so $f$ selects which member of that family is returned; it does not change the
family.  \qed

\paragraph{Correspondence with PrismQuant.}
Theorem~1 of \citet{prismquant} states, for a Gaussian-mixture source with
component weights $\pi_c$, that the rate--distortion-optimal allocation is
characterized by a single water level common to all components, because $\pi_c$
multiplies both the rate and the distortion contribution of component $c$ and
cancels in the stationarity condition.  Proposition 1 in the main paper is
the discrete, operational form of the same fact: the components are experts, the
weights are routing frequencies, the eigenmodes are replaced by a finite tier
set, and the water level is $\lambda$.  Nothing in the argument depends on
Gaussianity or on a closed-form distortion.

\paragraph{The promotion order (Corollary 1 in the main paper).}
With $\mathcal M$ ordered by increasing rate, expert $e$ is promoted from $M$ to
the next precision $M'$ as $\lambda$ decreases through
$\lambda_e(M\!\to\!M')=[d_e(M)-d_e(M')]/[R(M')-R(M)]$, the exchange rate between
distortion and bit-width along that step.  The order in which experts are
promoted as the budget grows is the order of these thresholds, and no $f_e$
appears in any of them (Figure 2 in the main paper).  A tier that lies above the
lower convex hull of its expert's distortion--rate points is never an $\arg\min$ and
is skipped; the thresholds are then taken between consecutive hull vertices rather
than adjacent tiers, and still contain no $f_e$.

\paragraph{Where frequency does act (the shadow price).}
Suppose the objective carries a bit-width term whose weight $w_e$ is not
proportional to the distortion weight $f_e$, next to the read term weighted by the
normalized decode count $\tilde f_e$,
\[
  J_e(M)=f_e\,d_e(M)+\lambda_{\mathrm s}w_e\,R(M)
         +\lambda_{\mathrm b}\tilde f_e\,R(M),
  \qquad \tilde f_e=f^{\mathrm{dec}}_e\big/\!\sum\nolimits_{e'}f^{\mathrm{dec}}_{e'} .
\]
Dividing by $f_e>0$ leaves the $\arg\min$ unchanged, so
\[
  \frac{J_e(M)}{f_e}=d_e(M)+\pi_e R(M),
  \qquad
  \pi_e=\frac{\lambda_{\mathrm s}w_e}{f_e}+\frac{\lambda_{\mathrm b}\tilde f_e}{f_e},
\]
and the assignment depends on the frequencies only through the two ratios
$w_e/f_e$ and $\tilde f_e/f_e$.  Two special cases recover the statements of the
main paper.  If $w_e\equiv1/E$ --- the stored and prefill functional of
Eq.~(2) in the main paper --- the first ratio is $1/(Ef_e)$, decreasing in
$f_e$, so a frequently routed expert reaches a given precision at lower marginal
cost and the assignment follows routing frequency (Figure~\ref{fig:two-channels}a).  If in addition
$f^{\mathrm{dec}}_e\propto f_e$, the second ratio is constant, and setting
$\lambda_{\mathrm s}=0$ returns Proposition 1 in the main paper exactly.

\paragraph{The cross-domain case.}
Dividing Eq.~(3) in the main paper by $F_e>0$ leaves the $\arg\min$
unchanged, so the total frequency of an expert cancels exactly as in
Proposition 1.  The bracket, however, is an average of the per-domain curves with
expert-specific weights $\pi_{d\mid e}$, and two experts with the same curves but
different domain mixtures face different averaged curves and can take different
tiers at one $\lambda$ (Figure~\ref{fig:two-channels}b).  If every expert had
the same mixture the bracket would be one fixed average for all experts and
Proposition 1 would apply verbatim.

\begin{figure}[h]
\centering
\includegraphics[width=\linewidth]{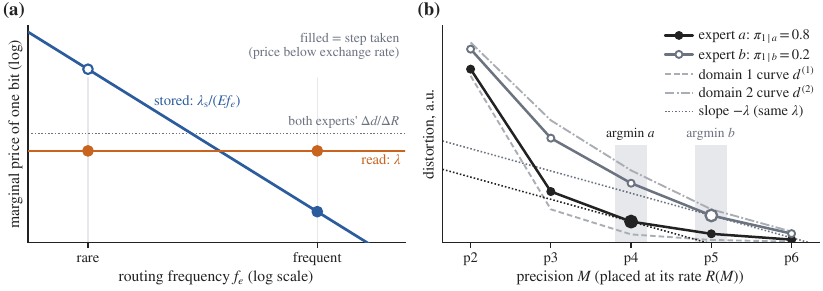}
\caption{Where routing frequency survives (Proposition 2 in the main paper);
illustrative values.  \textbf{(a)}~Divided by $f_e$, a bit costs $\lambda_{\mathrm s}/(Ef_e)$
under the stored functional but a constant $\lambda$ under the read one.
\textbf{(b)}~Two experts with the same per-domain curves but different mixtures
$\pi_{d\mid e}$ take different tiers at one $\lambda$.}
\label{fig:two-channels}
\end{figure}

\paragraph{The mirror-image failure.}
Dropping $f_e$ from the distortion while keeping it in the bit-width term gives
$J_e(M)=d_e(M)+\lambda\tilde f_e R(M)$, whose marginal cost tends to zero for an
expert the calibration stream never routes.  Such experts are then promoted to
the deepest precision at no charge, so the rate the solver books is not the rate
the artifact costs.  The two asymmetries are therefore not interchangeable:
removing $f$ from the bit-width term makes the assignment follow routing
frequency, and removing it from the distortion term makes the assignment follow
the absence of it.  The tie rule of Section 3.2 in the main paper is the
operational guard against the second.

\paragraph{Why no integer program is needed.}
Where every expert's tier family is convex, the per-expert Lagrangian of
Eq.~(6) in the main paper attains the global optimum over the assignments a
multiplier can reach, and a hull-restricted exact solver returns the same table.  A
dominated tier, one lying on or above the chord joining its two neighbours in the
(rate, distortion) plane, is a point a Lagrangian sweep cannot return and an exact
integer program could.  Between the multiplier's breakpoints the bisection returns
the feasible assignment nearest the budget rather than one exactly at it; every
budget in this paper is met to $0.001~\bpw$, and the residual is reported as part of
each table's rate rather than assumed away.  A tier family that is \emph{not}
convex leaves room an allocator can take, which is where the earlier-generation
gains of Appendix~\ref{app:hif} came from.

\section{Protocol, Rate Definitions and Secondary Scores}
\label{app:protocol}

\paragraph{The stored stream, in full.}
Section 3 in the main paper states only what the allocator consumes; the full
layout, following \citet{stagq}, is as follows.  Each group of $32$ columns stores one base pair $(lo,s)$, and each
weight a two-bit base code $q\in\{0,1,2,3\}$ followed by up to four one-bit
refinement planes $b_3,\dots,b_6$ held in plane order rather than interleaved per
scalar, so the planes above a selected depth form a contiguous suffix that need
not be fetched.  Plane $k$ records the sign of the residual its prefix left and
moves the reconstruction by $\pm s/2^{k-1}$.  Writing $u_M$ for the integer formed
by the bits read through tier $M$, the bulk reconstruction is one affine map,
\begin{equation}
  \widehat W^{(M)} = lo_M + s_M\,u_M,
  \qquad
  u_M = q\,2^{M-2}+\sum\nolimits_{k=3}^{M} b_k\,2^{M-k},
  \label{eq:bulk-affine}
\end{equation}
whose two constants $s_M=s/2^{M-2}$ and $lo_M=lo-s(2^{M-2}-1)/2^{M-1}$ follow in
closed form from the one stored pair; changing an expert's precision therefore
changes a prefix length and two derived constants, with no re-encoding, no
re-fitting and no reordering of weights (Figure~\ref{fig:stream}).  Besides the bulk fields the checkpoint
carries AWQ-style per-channel scaling, two 16-entry codebooks per matrix that
quantize the group metadata (one in $\log_2 s$, one in $lo/s$), and a sparse side
record holding the few weights the grid serves worst, which in this checkpoint is one
fixed-size entry per output row.  The side record is common to every tier, and every
rate and every quality number in this paper includes all of it.  A zero residual is
coded as a positive sign, so each refinement plane costs exactly one bit per
weight~\citep{stagq}, and every configuration sits on one scale.  The measurements of
Appendix~\ref{app:earlier} run on two earlier instances of the same
five-precision design whose audited ladders differ from each other by at most
$0.014~\bpw$; their
numbers are never mixed with the main tables.

\begin{figure}[h]
\centering
\includegraphics[width=\linewidth]{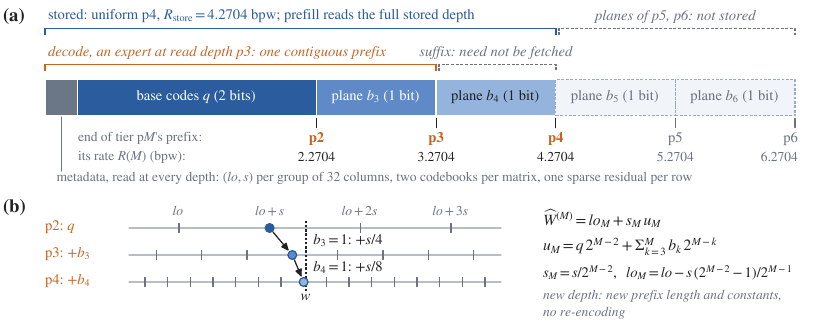}
\caption{The stored stream (schematic).  \textbf{(a)}~One expert matrix in plane
order, widths proportional to bits per weight: metadata read at every depth, base
codes $q$, then planes $b_3,\dots,b_6$; tier p$M$ is a contiguous prefix ending at
$R(M)$.  \duorate stores p4 and decode reads a p2, p3 or p4 prefix.
\textbf{(b)}~Equation~\eqref{eq:bulk-affine}: plane $k$ moves the reconstruction by
$\pm s/2^{k-1}$.}
\label{fig:stream}
\end{figure}

\begin{table}[h]
\centering
\caption{The floor of our ladder, recomputed field by field from the released base
checkpoint over the $28{,}991{,}029{,}248$ expert weights.  \emph{Measured}: count
from the checkpoint; \emph{spec}: word length from the format specification.  Rows are rounded
individually; the total is summed before rounding.}
\label{tab:ownfields}
\small
\begin{tabular}{lrl}
\toprule
Field & \bpw & Provenance \\
\midrule
Base code, two bits per weight      & 2.000000 & measured \\
Group record (\emph{lo}, \emph{s}), 8\,bits per group of 32 & 0.250000 & measured \\
Two 16-entry codebooks per matrix   & 0.000326 & measured count, spec word \\
Sparse residual record, one per row & 0.020074 & measured count, spec word \\
\midrule
\textbf{Total, p2}                  & \textbf{2.270399} & \\
Published floor                     & 2.2704 & \\
Difference                          & $-7\times10^{-7}$ & \\
\bottomrule
\end{tabular}
\end{table}

One field named in the layout above is absent from the table.  The per-channel
scaling of the gate and up projections is made layer-shared by the encoder so that
its reciprocal folds into the preceding normalization at deployment, which is why it
costs the expert segment nothing; only the down projection keeps a per-expert
vector, and charging that at half precision would add $0.00260~\bpw$ to the floor.
The residual of $7\times10^{-7}$ against the published floor is not an
error in either direction: the published floor rounds the rate of each of the $18{,}432$
matrices to four decimals before averaging, and this audit sums the unrounded fields.  Every
rate in this paper uses the published floor, so both the omission and the rounding
are constant across our configurations and cancel in every comparison among them.

\paragraph{Three rate definitions.}
Three rates are distinguished throughout and never interchanged.
The \emph{stored} rate of a table is the unweighted mean of $R(M_e)$ over its
$6{,}144$ experts, \Rstore{} of Eq.~(2) in the main paper: $4.2704~\bpw$ for the
resident p4 tier and, for a read table, the footprint it would have if it were kept
resident.  The \emph{budgeted read}
rate is $\sum_e f^{\mathrm{cal}}_e R(M_e)/\sum_e f^{\mathrm{cal}}_e$ with
$f^{\mathrm{cal}}$ a prefill-granularity routing count on the served domain.  The
\emph{realized read} rate replaces $f^{\mathrm{cal}}$ by the decode-step routing
count collected on the same domain's generations.  One subtlety: the allocator's
bisection meets its target against the frequencies carried in the calibration
activation table, and it does so to within $0.001~\bpw$, while
Table~\ref{tab:ledger} recomputes the same rate from the independent routing
dump where one exists; the two populations are close but not identical, which is
why the GSM8K row at $2.9~\bpw$ reads $2.879$ where the solver records
$2.8999$.  Both are the same definition and neither is a decode rate.  All three
are computed from the table's tier histogram and the ladder of Eq.~(1) in the main
paper, never taken from the nominal budget: in an earlier experiment, a table budgeted
under one benchmark's routing and served under another's read $5.5~\bpw$ against a
$3.6~\bpw$ budget.  Per-token traffic follows from the read rate: a decode step
activates $48\times8$ experts of $4{,}718{,}592$ weights each, $1.81\times10^{9}$
weights, so a $2.9~\bpw$ read corresponds to $657$~MB per token.  Table~\ref{tab:storeread} lists the stored and read rates of the
\duorate rows of Table 1 in the main paper alongside the uniform tiers.

\begin{table}[h]
\centering
\caption{Stored and read rates of the uniform tiers of our ladder and of the \duorate
rows of Table 1 in the main paper, with BF16 and \bitsmoe for reference.  A uniform
tier stores and reads the same width; \duorate stores p4, and its \emph{Read} is the
budget each benchmark's table was solved to.  Uniform p4 was not scored on GSM8K.}
\label{tab:storeread}
\small
\resizebox{\linewidth}{!}{%
\begin{tabular}{l|cc|c|ccc|c}
\toprule
\textbf{Method} & \textbf{Store} & \textbf{Read} & \textbf{PPL}$\downarrow$
& \textbf{MMLU-Pro} & \textbf{MMLU} & \textbf{GSM8K} & \textbf{Avg.} \\
\midrule
BF16              & 16.000 & 16.000 & 8.27  & 56.86 & 83.30 & 91.25 & 77.14 \\
\midrule
Uniform p2        & 2.2704 & 2.2704 & 10.08 & 52.64 & 80.00 & 88.00 & 73.55 \\
Uniform p3        & 3.2704 & 3.2704 & 8.69  & 55.14 & 81.96 & 88.50 & 75.20 \\
Uniform p4        & 4.2704 & 4.2704 & 8.37  & 57.21 & 82.67 & \na   & \na   \\
\midrule
\duorate          & 4.2704 & 2.900  & 9.07  & 56.00 & 81.47 & 91.00 & 76.16 \\
\duorate          & 4.2704 & 2.600  & 9.46  & 53.86 & 80.56 & 89.50 & 74.64 \\
\duorate          & 4.2704 & 2.300  & 9.97  & 52.64 & 80.00 & 88.25 & 73.63 \\
\midrule
\bitsmoe 2-bit    & 2.5425 & 7.5008 & 14.69 & 46.21 & 77.75 & 78.00 & 67.32 \\
\bottomrule
\end{tabular}%
}
\end{table}

\begin{table}[h]
\centering
\caption{Rates of the $2.6$ and $2.9~\bpw$ read tables, recomputed from their tier
histograms and the ladder of Eq.~(1) in the main paper.  $R_{\mathrm{store}}$ is the
unweighted rate of the read table itself, not of the resident p4 tier;
$R^{\mathrm{cal}}_{\mathrm{read}}$ is the budgeted rate under prefill routing;
$R^{\mathrm{dec}}_{\mathrm{read}}$ uses decode-step routing and exists only for
GSM8K, since the log-likelihood benchmarks have no decode phase.}
\label{tab:ledger}
\small
\begin{tabular}{llrrrr}
\toprule
Domain & Budget & $R_{\mathrm{store}}$ & $R^{\mathrm{cal}}_{\mathrm{read}}$
& $R^{\mathrm{dec}}_{\mathrm{read}}$ & $R^{\mathrm{dec}}-R^{\mathrm{cal}}$ \\
\midrule
GSM8K & 2.6 & 2.464 & 2.589 & 2.559 & $-0.029$ \\
      & 2.9 & 2.691 & 2.879 & 2.818 & $-0.061$ \\
\midrule
MMLU  & 2.6 & 2.567 & 2.587 & --- & --- \\
      & 2.9 & 2.822 & 2.880 & --- & --- \\
WikiText-2 & 2.6 & 2.659 & 2.600 & --- & --- \\
           & 2.9 & 2.917 & 2.899 & --- & --- \\
MMLU-Pro   & 2.6 & 2.546 & 2.600 & --- & --- \\
           & 2.9 & 2.809 & 2.900 & --- & --- \\
\bottomrule
\end{tabular}
\end{table}

The stored rate of a read table is $0.02$--$0.19~\bpw$ below its budgeted read
rate on every row except WikiText-2, where it is $0.02$--$0.06$ above; the experts a
workload never routes to sit at the lowest tier and weigh nothing in the read
functional, which is what pulls the stored rate down.  On GSM8K the realized decode
traffic is a further $0.03$--$0.06~\bpw$ below the budgeted rate, and $0.04$--$0.08$
below the nominal budget, because decode routes a somewhat different population than
the prompts the budget was set on.  The GSM8K and MMLU rows recompute
$R^{\mathrm{cal}}_{\mathrm{read}}$ from an independent routing dump, while the
WikiText-2 and MMLU-Pro rows carry the frequency the allocator was solved against,
which is why they land on their targets.  The share of decode traffic falling on
experts a workload's own prompts never reached is $0.48\%$ on GSM8K and rises to
$20\%$ only for MMLU, whose prompts are $36$--$89$ tokens long: short prompts are the
least informative calibration.

\paragraph{Instrument resolution.}
Perplexity is a continuous quantity paired per document, and a paired sign test
over $62$ documents separates configurations that accuracy over $1{,}400$
questions cannot.  A whole tier step of $1.000~\bpw$ is worth $-13.74\%$
perplexity ($62$ of $62$ documents) but only $+2.50$ MMLU-Pro points ($p=0.017$),
so accuracy cannot resolve an effect of about a tenth of a bit, the size Section 4.2
in the main paper measures for the allocation.
Accuracy differences below about one point are therefore reported as ties unless a
paired test says otherwise, and the rate order of the allocated configurations is
read from perplexity.

\paragraph{Which statistic the intervals are computed on.}
Word perplexity is $\exp(-\sum_i \ell_i / \sum_i n_i)$ over documents, a ratio of
sums rather than a mean of per-document values, so a comparison has two natural
summaries: the relative change of that aggregate, and the mean of the per-document
relative changes.  They are not close on this corpus, because the $62$ documents
differ in length by more than an order of magnitude and the aggregate is
token-weighted.  Every perplexity interval in this paper is a percentile bootstrap over
documents computed on the \emph{same} statistic as the point estimate beside it,
namely the aggregate, with $20{,}000$ resamples at seed $42$; ties are excluded
from the sign test and reported separately.  Figure~\ref{fig:waterfall} shows the
per-document changes behind three of the comparisons: the aggregate gains are not
carried by a few documents.

\begin{figure}[h]
\centering
\includegraphics[width=\linewidth]{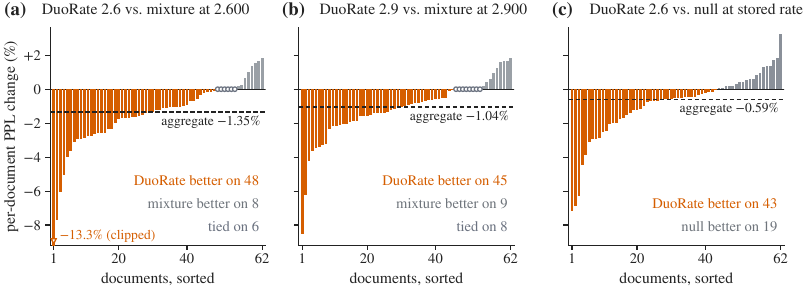}
\caption{Per-document perplexity change of \duorate on WikiText-2, sorted over the
$62$ documents.  \textbf{(a, b)}~Against a random mixture of two adjacent tiers at the
table's read rate (least favourable draw); \textbf{(c)}~against the interpolated null
at the stored rate, the weakest comparison in the upper block of
Table~\ref{tab:nullblock}.  Dashed: the aggregate change;
open circles: exact ties, excluded from the sign test.}
\label{fig:waterfall}
\end{figure}

\paragraph{The interpolated uniform-tier null.}
Table~\ref{tab:nullblock} compares a mixed table against a synthetic control
rather than against a run, which is what lets it cover both rate definitions,
where a real mixture exists only at the rates we materialized.  For a mixed table at rate $r$ lying
between the uniform tiers at rates $r_{\mathrm{lo}}<r<r_{\mathrm{hi}}$, the null
assigns document $i$ the log-likelihood
$\ell^{\mathrm{null}}_i=\ell^{\mathrm{lo}}_i+\alpha(\ell^{\mathrm{hi}}_i-\ell^{\mathrm{lo}}_i)$
with $\alpha=(r-r_{\mathrm{lo}})/(r_{\mathrm{hi}}-r_{\mathrm{lo}})$, so it is what
the same format would deliver at that rate if precision could be bought
continuously and the allocation carried no information.  Because $\alpha$ is fixed
by the rate, the two rate definitions place the null at two different points and give
two different tests, which is why both are reported.  The comparison is per
document, so the sign test and the document bootstrap apply unchanged.  On the two
budgeted cells where a real mixture was also run, the two controls agree in sign and
within $0.11$ to $0.14$ percentage points ($1.35$ against $1.46\%$, $1.04$ against $1.18\%$);
that agreement is a calibration of the synthetic null against the real control, not
a substitute for running it.

\begin{table}[h]
\centering
\caption{The WikiText-2 read tables against an interpolated uniform-tier null at
their own rate, placed at the budgeted read rate and at the stored rate.  Lower block: tables re-solved from the
WikiText-2 validation split.  Intervals are percentile bootstraps on the aggregate
perplexity.}
\label{tab:nullblock}
\small
\begin{tabular}{ll|rl|r}
\toprule
\textbf{Arm} & \textbf{Rate definition} & \textbf{$\Delta$} & \textbf{Wins\,/\,losses, 95\% CI} & \textbf{$p$} \\
\midrule
Read $2.6$ & budgeted & $-1.46\%$ & 57 / 5,  $[-1.90,-1.05]$ & $3.1\times10^{-12}$ \\
           & stored   & $-0.59\%$ & 43 / 19, $[-1.06,-0.11]$ & $0.0032$ \\
Read $2.9$ & budgeted & $-1.18\%$ & 48 / 14, $[-1.64,-0.74]$ & $1.7\times10^{-5}$ \\
           & stored   & $-0.92\%$ & 46 / 16, $[-1.39,-0.47]$ & $0.00018$ \\
\midrule
\multicolumn{5}{l}{\emph{with the calibration split disjoint from the scored documents}} \\
Read $2.6$ & budgeted & $-1.28\%$ & 56 / 6,  $[-1.67,-0.90]$ & $3.0\times10^{-11}$ \\
           & stored   & $-0.46\%$ & 41 / 21, $[-0.90,-0.01]$ & $0.0151$ \\
Read $2.9$ & budgeted & $-0.97\%$ & 49 / 13, $[-1.43,-0.53]$ & $4.8\times10^{-6}$ \\
           & stored   & $-0.77\%$ & 45 / 17, $[-1.23,-0.31]$ & $0.0005$ \\
\bottomrule
\end{tabular}
\end{table}

\paragraph{Paired tests behind Figure 4 in the main paper, and the adjacent tiers.}
Table~\ref{tab:where} lists the mixture comparisons of Figure 4 in the main paper with
their win/loss counts, and compares the $2.9~\bpw$ table with the uniform tiers on
either side of it.  Against uniform p2, which reads $0.63~\bpw$ less, it wins on all
$62$ WikiText-2 documents and by $3.36$ MMLU-Pro points ($p=3.7\times10^{-4}$), the only one
of the table's six accuracy comparisons that survives a Bonferroni correction.  Against
uniform p3, which reads $0.37~\bpw$ more, it ties on all three accuracy benchmarks (the
largest difference, $+2.50$ on GSM8K, has $p=0.12$) while trailing by $4.39\%$ in
perplexity.

\begin{table}[h]
    \centering
    \small
    \caption{Paired comparisons of the \duorate read tables against controls in the
    same format.  Top: a random mixture of two adjacent tiers at the table's read
    rate and at a dearer one (worst of the draws).  Bottom: the uniform tiers below
    and above the $2.9~\bpw$ table.  Perplexity rows report document wins/losses, a
    bootstrap interval and an exact sign test; accuracy rows report exact McNemar.}
    \label{tab:where}
    \begin{tabular}{ll|rl|r}
        \toprule
        \textbf{Comparison} & \textbf{Benchmark / rate} & \textbf{$\Delta$} & \textbf{Wins\,/\,losses, 95\% CI} & \textbf{$p$} \\
        \midrule
        \multicolumn{5}{c}{\emph{against a random mixture of the same tiers}, WikiText-2} \\
        \midrule
        \duorate $2.6$ & mixture at $2.600$ & $-1.35\%$ & 48 / 8,  $[-1.76,-0.94]$ & $4.7\times10^{-8}$ \\
                       & mixture at $2.700$ & $-0.07\%$ & 32 / 23, $[-0.55,+0.47]$ & $0.28$ \\
        \duorate $2.9$ & mixture at $2.900$ & $-1.04\%$ & 45 / 9,  $[-1.43,-0.64]$ & $7.3\times10^{-7}$ \\
                       & mixture at $2.950$ & $-0.28\%$ & 34 / 21, $[-0.73,+0.24]$ & $0.10$ \\
        \midrule
        \multicolumn{5}{c}{\emph{against uniform p2} ($2.27~\bpw$)} \\
        \midrule
        \duorate $2.9$ & WikiText-2 & $-9.94\%$ & 62 / 0 & $4.3\times10^{-19}$ \\
                       & MMLU-Pro   & $+3.36$ pt    & 108 / 61 & $0.00037$ \\
                       & MMLU       & $+1.47$ pt    &  & $0.059$ \\
                       & GSM8K      & $+3.00$ pt    &  & $0.10$ \\
        \midrule
        \multicolumn{5}{c}{\emph{against uniform p3} ($3.27~\bpw$)} \\
        \midrule
        \duorate $2.9$ & WikiText-2 & $+4.39\%$ & 2 / 59, $[+3.47,+5.60]$ & $1.6\times10^{-15}$ \\
                       & MMLU-Pro   & $+0.86$ pt    &  & $0.33$ \\
                       & MMLU       & $-0.49$ pt    &  & $0.49$ \\
                       & GSM8K      & $+2.50$ pt    &  & $0.12$ \\
        \bottomrule
    \end{tabular}
\end{table}

\paragraph{How often the stage-one cap binds.}
Section 3.2 in the main paper restricts stage two to tiers at or below the
stored p4 tier.  Solving the WikiText-2 and MMLU-Pro read tables twice, once with
that restriction and once over p2--p6, the cap clips $15$ and $29$ experts for the
WikiText-2 tables at $2.6$ and $2.9~\bpw$ and $12$ and $17$ for MMLU-Pro, and changes
$23$ to $37$ of the $6{,}144$ tier assignments in all.

\paragraph{Calibration and evaluation items.}
Each read table is solved from the served benchmark's own prefill, so the
calibration domain and the scored domain coincide by construction.  Whether the
calibration \emph{items} coincide is a separate question, and it is audited rather
than assumed.  GSM8K calibrates on the training split, MMLU on the
validation split, and MMLU-Pro on test items the stratified $14\times100$ sample
does not draw; the builder refuses a calibration split equal to the evaluation
split outright.  Comparing every scored question against the calibration corpus as
normalized text, no scored GSM8K item appears in it, and of $1{,}397$
MMLU-Pro items three share a leading instruction block with it but none appears in
full.  MMLU has one true repeat in $1{,}419$, a question the dataset itself carries in both
its validation and its test split.
WikiText-2 is the exception: its activation moments were collected on the same
$62$ test documents that are scored, so for that benchmark we also solve the table
from the validation split and report both.

\paragraph{Secondary answer filter.}
Table 1 in the main paper reports GSM8K under the strict filter.  Under the
flexible-extract filter the same runs score $91.00$ and $89.00$ for \duorate at
$2.9$ and $2.6~\bpw$, $87.00$ and $89.50$ for uniform p2 and p3, and $78.50$ for
\bitsmoe, within one point of the strict scores; the paired tests of
Table~\ref{tab:where} use the strict filter.

\paragraph{Question counts.}
Accuracies quoted inside paired tests are computed over all scored questions
($n=1{,}425$ for MMLU, $1{,}400$ for MMLU-Pro, $400$ for GSM8K); the two exact
repeats in the $25$-per-subject MMLU sample are paired as separate items.

\paragraph{Tier histograms.}
The stored table is the uniform p4 tier on all $6{,}144$ experts
(Figure~\ref{fig:tierhist}).  The GSM8K read table
at a $2.9~\bpw$ budget assigns $3{,}861$ / $1{,}979$ / $304$ experts to
p2 / p3 / p4; the WikiText-2 tables are $3{,}973$ / $1{,}952$ / $219$ at a
$2.6~\bpw$ budget and $2{,}792$ / $2{,}733$ / $619$ at $2.9$.

\begin{figure}[h]
\centering
\includegraphics[width=\linewidth]{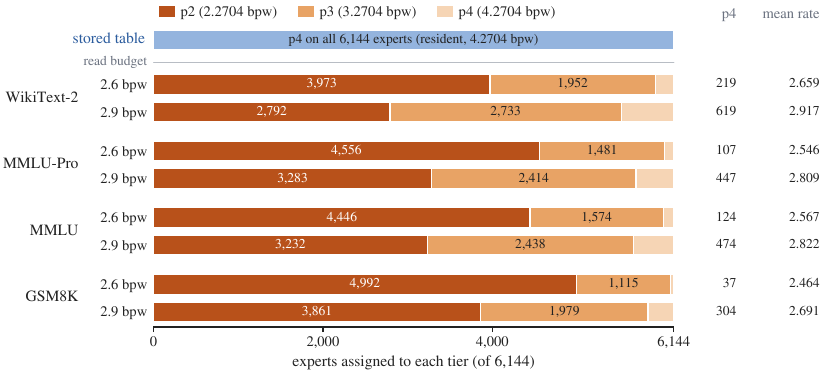}
\caption{Tier histograms of the $2.6$ and $2.9~\bpw$ read tables.  Each
bar splits the $6{,}144$ experts into p2 / p3 / p4; every read tier is a byte prefix of
the stored p4 table (top).  Right: the p4 count and the unweighted mean of $R(M_e)$,
the $R_{\mathrm{store}}$ column of Table~\ref{tab:ledger}.}
\label{fig:tierhist}
\end{figure}

\paragraph{Where the read budget goes.}
Figure~\ref{fig:layers} shows how the read tables spend their budget.  No layer-level
constraint is imposed, yet on every benchmark the first layers stay close to p2, the
rate rises toward the middle of the network, and the last layer is read deepest on
the three log-likelihood benchmarks.  Within a layer, the routing-weighted rate, which
is what decode pays, runs above the unweighted rate on MMLU-Pro, MMLU and most clearly
GSM8K: the deeper tiers go to the experts the workload actually routes to, while the
rarely routed experts of the same layer stay shallow, consistent with the
long-decoding result of Section 4.1 in the main paper.

\begin{figure}[h]
\centering
\includegraphics[width=1.00\linewidth]{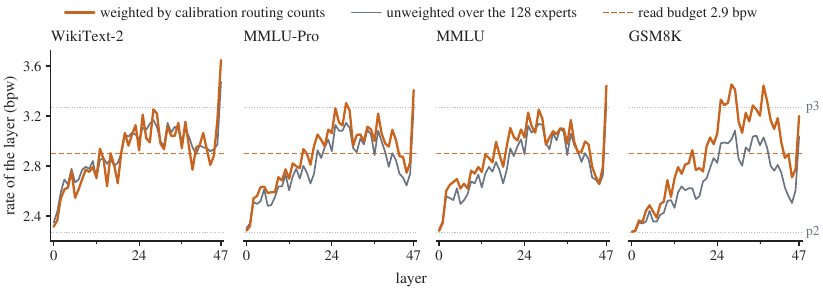}
\caption{Rate of each of the $48$ layers in the $2.9~\bpw$ read tables, weighted by
the calibration routing counts (what a decode step pays) and unweighted over the
$128$ experts of the layer.  Dotted lines mark the p2 and p3 tiers; the dashed line
is the read budget.}
\label{fig:layers}
\end{figure}

\paragraph{Normalization of pooled profiles.}
The output weights $A$ of the four pooled domains differ in global scale by a factor
of $45$, because the measured standard deviation $\sigma$ of the quantization-induced
margin perturbation differs per domain, and the per-domain distortion scales span four
orders of magnitude.  Each domain is therefore divided
by one global scalar before pooling.  Normalizing each expert's channel vector to
unit mean instead would also remove the difference in activation energy between
experts, which is part of the signal the allocator ranks by.

\paragraph{Coverage and the measured $\sigma$.}
The pooled output profile leaves $87$ of $18{,}432$ matrices without a measured
$A$ and $12$ without a measured $X$, against $1{,}482$ and $1{,}857$ for a
single-domain profile; a matrix without $A$ receives the median of the mean $A$
over matrices that have one, never $1$, since a unit weight would inflate its
distortion by four orders of magnitude.  On MMLU calibration prompts at
$3.2~\bpw$ the median top-two margin is $2.125$ and the quantization-induced
perturbation of the margin has standard deviation $\sigma=1.2525$; the realized
$\arg\max$ flip rate is $9.53\%$ against the Gaussian model's $14.03\%$.  On
GSM8K solutions the median margin is $9.75$ with $\sigma=3.956$ and a mean shift
of $-3.62$; including the shift changes $2.4\%$ of assignments and no measured
score, because it acts as a near-uniform multiplier that the Lagrangian absorbs
into $\lambda$.  One further observation bounds every prefill-side calibration:
re-running the prompt alone rather than the full sequence changes $1.27\%$ of
routing decisions, because sequence length changes the BF16 reduction order and
flips near-ties at the top-$k$ boundary.

\section{Audit of the External Baseline}
\label{app:bitsmoe}

\bitsmoe~\citep{bitsmoe} is the one published per-expert MoE allocator on this
base model whose weights are downloadable, so we re-score its released checkpoint
rather than re-implement the method.  Placing it on our rate axis
requires recomputing its rate from the artifact, because the label on a
checkpoint is not a rate.  This appendix reports that recomputation, which is
arithmetic on published equations and on the released tensors and is reproducible
from the checkpoint alone, and then the cross-protocol comparison.

\paragraph{What the released 2-bit checkpoint stores.}
Summing the \texttt{safetensors} headers of the released checkpoint over
$184{,}818$ tensors and dividing by the $28{,}991{,}029{,}248$ expert weights of
Qwen3-30B-A3B gives the per-field breakdown in bits per expert weight: quantized
payload $1.9583$; FP16 shared basis $0.3333$; group-128 scales $0.1240$; the
spectral-energy vector $0.0207$; rank indices $0.0413$; a bit-identical redundant
copy of those indices $0.0413$; container metadata $0.0156$; three
super-experts kept in FP16 $0.0078$.  The expert segment therefore stores
$2.5425~\bpw$ (the fields are rounded independently; the exact total is
$2.54246$), or $2.5494$ including the router, against a paper label of $2.0$
and an intermediate $2.2917$ under that paper's own
Equations~(96)--(98), which charge the payload and the FP16 basis but no side
channel.  Two independent checks confirm the $2.5425$ figure rather than the
label.  First, $2.5425~\bpw$ over the resident expert weights is $8.581$~GiB,
which reproduces the $8.58$ and the $6.29\times$ saving that paper reports in its
own efficiency table, whereas a true $2.0~\bpw$ would be $6.750$~GiB and
$8.00\times$; the same arithmetic identifies that table's unit as GiB rather than
the GB its header states, since BF16 comes to $54.000$~GiB and $57.98$~GB
against a reported $54.00$.  Second, its per-projection budget equation
$B_h=\lfloor E n\,b_{\mathrm{eq}} - 16n^2/m\rfloor$ evaluates to $436{,}906$ at
$b_{\mathrm{eq}}{=}2$, while every one of the $144$ blocks in the released
checkpoint carries exactly $513{,}365$ with zero variance.  That value is the
same expression with the reservation term's $m$ replaced by $3n$
($\lfloor 16n^{2}/(3n)\rfloor=10{,}922$ against $\lfloor 16n^{2}/m\rfloor=87{,}381$),
an eight-fold smaller
allowance for the shared basis and $17.5\%$ more budget than the stated equation
grants; charging the basis at $b_{\mathrm{eq}}$ rather than at $16$ bits is a
second, numerically indistinguishable reading.  The released code is
inference-only, so this is an observation about the artifact, and we do not
attribute it to a particular line.

\begin{figure}[h]
\centering
\includegraphics[width=\linewidth]{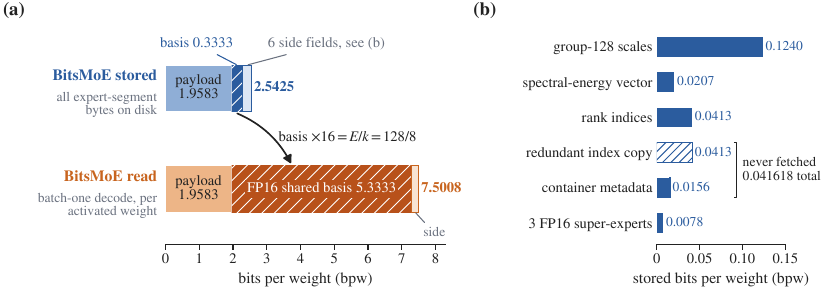}
\caption{Per-field rates of the released \bitsmoe 2-bit checkpoint, recomputed from
its safetensors headers.  \textbf{(a)}~The expert segment stores $2.5425~\bpw$; a
batch-one decode step charges the FP16 shared basis to the $k{=}8$ selected experts
rather than all $E{=}128$, and with the never-fetched fields refunded reads
$7.5008~\bpw$.  \textbf{(b)}~Stored side fields; hatched parts ($0.041618~\bpw$) are
never fetched.}
\label{fig:bitsmoe-fields}
\end{figure}

\paragraph{Why its stored rate and its read rate differ by $2.95\times$ (Figure~\ref{fig:bitsmoe-fields}).}
The shared basis is one FP16 matrix per layer and projection, shared by all
$E{=}128$ experts of that layer.  Charged against storage it costs
$16n/(E m)=0.3333~\bpw$.  Charged against a batch-one decode step, which touches
only the $k{=}8$ selected experts but must still project through the whole basis,
it costs $16n/(k m)=5.3333~\bpw$, an inflation of exactly $E/k=16$.  That the
projection is a per-layer and not a per-expert cost is stated by
\citet{bitsmoe} itself: its projection ``is computed once per MoE layer and reused
across routed experts''.  The released code matches, computing the gate and up
bases as unconditional dense products of the layer input before any routing is
applied, and the down basis as one dense product outside the packed kernel.

The read rate charges only the fields that kernel actually fetches.  The
checkpoint's \texttt{original\_indices}, \texttt{segments},
\texttt{original\_rank} and \texttt{groupsize} never reach it, and together they
are $0.041618~\bpw$: the released packer separates ``runtime-facing buffers'' from
``high-level mapping metadata'', the kernel's argument list contains only the
payload, the rank indices, the tile and slab metadata, the scales and the energy
vector, and \texttt{original\_indices} is written as a clone of the rank indices.
Refunding those four leaves a read rate of $7.5008~\bpw$ against a stored
$2.5425$, so the read field set is a strict subset of the stored one.  Per decoded
token and per MoE layer the basis moves $25.17$~MB and the eight selected experts'
private segments move $10.23$~MB, so the basis is $71.1\%$ of the traffic and the
total is $35.39$~MB against FP16's $75.50$~MB, a $2.13\times$ reduction where the
resident footprint fell by $6.29\times$.  Against our uniform p2 tier, which stores
$0.27~\bpw$ less, it moves $3.30\times$ more.  A formulation
that constrains only the unweighted functional has no term that can see the
difference, and none of the published allocators in Table~\ref{tab:audit} has one.  The basis
term scales as $E$ over the number of distinct experts a step touches, so it
shrinks as the batch grows and disappears only when a step reaches all $128$; the
figures here describe the batch-one decode regime that method measures in its own
efficiency table.

\begin{table}[h]
\centering
\caption{The published Qwen3-30B-A3B-Base results of \citet{bitsmoe} (upper
half, C4 perplexity, zero-shot) with our rows under our protocol (lower half,
WikiText-2 perplexity).  \textbf{The halves are not commensurable.}  \emph{Bits} is
the published label in the upper half; in the lower half it is store\,/\,read where
the two differ.}
\label{tab:crossprotocol}
\small
\resizebox{\textwidth}{!}{%
\begin{tabular}{l|c|c|ccccccc|c}
\toprule
\textbf{Method} & \textbf{Bits} & \textbf{PPL}$\downarrow$
& \multicolumn{7}{c|}{\textbf{Accuracy}$\uparrow$ (\%)} & \textbf{Avg.} \\
\cmidrule(lr){4-10}
& & & HellaS. & MathQA & MMLU & Openb. & WinoG. & GSM8K & HumanE. & \\
\midrule
\multicolumn{11}{c}{\emph{as published in \citet{bitsmoe}, their protocol}} \\
\midrule
FP16     & 16 & 10.24 & 81.35 & 60.03 & 78.77 & 45.00 & 72.85 & 83.47 & 56.10 & 68.22 \\
\cmidrule(lr){1-11}
HQQ      & 2  & 23.65 & 63.05 & 33.17 & 48.82 & 36.40 & 60.85 & 22.67 & 11.59 & 39.51 \\
GPTQ~\citep{gptq} & 2  & 15.63 & 70.16 & 24.89 & 39.17 & 39.40 & 60.62 & 4.32  & 0.00  & 34.08 \\
MiLo     & 2  & 21.53 & 62.82 & 31.83 & 44.51 & 35.40 & 59.98 & 19.64 & 7.93  & 37.44 \\
MoEQuant & 2  & 15.34 & 66.44 & 45.09 & 70.02 & 40.00 & 67.72 & 49.36 & 26.83 & 52.21 \\
\bitsmoe & 2  & 16.07 & 74.09 & 52.70 & 70.87 & 43.40 & 72.93 & 75.51 & 43.90 & \best{61.91} \\
\midrule
\multicolumn{11}{c}{\emph{ours, our protocol --- read down this half only}} \\
\midrule
BF16                            & 16     & 8.27  & \na & \na & 83.30 & 44.20 & 73.24 & 91.25 & 51.83 & \na \\
\bitsmoe, re-scored             & 2.5425 / 7.5008 & 14.69 & \na & \na & 77.75 & 43.20 & 72.14 & 78.00 & 35.37 & \na \\
\duorate $2.3$                  & 4.2704 / 2.300 & 9.97  & \na & \na & 80.00 & \na & \na & 88.25 & \na & \na \\
\duorate $2.9$                  & 4.2704 / 2.900 & 9.07  & \na & \na & 81.47 & \na & \na & 91.00 & \na & \na \\
\bottomrule
\end{tabular}%
}
\end{table}

\paragraph{Does our harness agree with theirs where it should?}
Re-scoring another group's checkpoint is only informative if the harness
reproduces their numbers on benchmarks where the two protocols already coincide.
Table~\ref{tab:anchor} runs that check.  On OpenBookQA and WinoGrande neither side
passes a few-shot count, so both take the harness default of zero, and both
benchmarks are scored on their full official splits with no subsampling on either
side; these are the four cells where nothing should differ.  They agree to within
$0.80$ points, with deviations in both directions, and the two unquantized rows,
where no quantization method is involved, agree as closely as the two quantized
ones.  We also
confirm from their reported accuracies that they use full test splits: each of
their figures divides back into an integer number of correct answers over the
official split size.

The two generative benchmarks behave differently, and consistently with what
greedy decoding on different hardware predicts.  Our HumanEval score for their
checkpoint is $8.53$ points below theirs and our GSM8K score is $2.49$ points
above; neither difference is significant ($p{=}0.11$ at $n{=}164$, $p{=}0.31$ at
$n{=}400$), and the opposite signs argue against a systematic bias.  A multiple-choice
item is scored by comparing the log-likelihood of a few fixed continuations, which
has no error-accumulation path; a generated program or derivation diverges
irrecoverably if one early token flips, and we decode in BF16 on an Ascend
910B3 where they decode in FP16 on an A100.  We therefore read the multiple-choice
cells as evidence that the harness is aligned and the generative cells as evidence
that $164$-question and $400$-question generative benchmarks cannot resolve
differences of this size.

\begin{table}[t]
    \centering
    \small
    \caption{Our re-scoring of the released \bitsmoe checkpoint and of the
    unquantized model against that paper's published numbers.  \emph{Aligned}: same
    shot count and same full split on both sides.  Our unquantized MMLU figure is the
    full test set; our 2-bit figure is the $25$-per-subject sample of Table 1 in the
    main paper.}
    \label{tab:anchor}
    \begin{tabular}{ll|ccc}
        \toprule
        \textbf{Benchmark} & \textbf{Arm} & \textbf{Published} & \textbf{Ours}
        & \textbf{$\Delta$} \\
        \midrule
        \multicolumn{5}{c}{\emph{aligned: 0-shot on both sides, full split on both sides}} \\
        \midrule
        OpenBookQA (acc\_norm) & unquantized & 45.00 & 44.20 & $-0.80$ \\
                               & 2-bit       & 43.40 & 43.20 & $-0.20$ \\
        WinoGrande (acc)       & unquantized & 72.85 & 73.24 & $+0.39$ \\
                               & 2-bit       & 72.93 & 72.14 & $-0.79$ \\
        \midrule
        \multicolumn{5}{c}{\emph{aligned, but generative and underpowered}} \\
        \midrule
        HumanEval (pass@1)     & 2-bit & 43.90 & 35.37 & $-8.53$ ($p{=}0.11$) \\
        GSM8K                  & 2-bit & 75.51 & 78.00 & $+2.49$ ($p{=}0.31$) \\
        \midrule
        \multicolumn{5}{c}{\emph{not aligned: they score 0-shot, we score 5-shot}} \\
        \midrule
        MMLU                   & unquantized & 78.77 & 81.12 & $+2.35$ \\
                               & 2-bit       & 70.87 & 77.75 & $+6.88$ \\
        \bottomrule
    \end{tabular}
\end{table}

\paragraph{Why our numbers for their checkpoint are not their numbers.}
The released checkpoint scores $70.87$ MMLU and $75.51$ GSM8K under its own
harness and $77.75$ / $78.00$ under ours.  Two protocol axes account for that, and
only one of them is a disagreement.

On MMLU the difference is real and it is the few-shot count.  We score 5-shot;
\bitsmoe passes no few-shot count --- its released evaluation configs carry no such
field and its launcher builds no such flag, so the harness default of zero applies
--- and its text lists MMLU among the zero-shot tasks.  The size of that axis shows
up on the arms where no quantization is involved at all: our BF16 run over the
full $14{,}042$ questions scores $81.12$ against their FP16 $78.77$.  A second,
smaller axis is ours: scoring $25$ questions per subject turns a size-weighted
benchmark into a subject-macro one, since MMLU's subjects differ in size by more
than $20\times$.  Re-weighting our per-subject accuracies by the official subject
sizes moves BF16 from $83.30$ to $81.22$ and the \bitsmoe arm from $77.75$ to
$76.85$; the BF16 estimate lands $0.10$ from our own full-set run, which is
consistent, though the re-weighted estimator's own standard error is about $1.5$
points, so the agreement should be read as a sanity check and not as a precision
claim.  On that common footing the quantization gap is $4.4$ points where theirs
is $7.9$.

On GSM8K there is nothing to explain.  We score the first $400$ questions and they
score all $1{,}319$, and a two-proportion test on that pair gives $p=0.31$: the two
numbers are not distinguishable before any protocol difference is even considered.
We do not correct our figure onto the full set.  This project does
hold same-arm runs at both sizes, but they were taken with different generation
caps that the result files do not record, so any correction estimated from them
would confound the subsample with the cap --- which is the kind of comparison this
paper exists to argue against.

Two further caveats.  Our window is $2{,}048$
tokens, and four of the $57$ MMLU subjects have 5-shot contexts longer than that,
so those subjects are left-truncated on our side and are not on theirs; moving to
their zero-shot setting would remove the truncation along with the few-shot gap.
And nothing makes the two columns strictly comparable in any case: the hardware and
dtype differ (BF16 on Ascend against FP16 on A100), and their number comes from
their own packed kernel while ours comes from a dequantized dense copy of the same
weights.  Their harness version is itself unsettled --- the paper names one release
while the code released with it pins a different, later one --- so we do not treat
it as an axis we could align to.  Table 1 in the main paper scores every arm under one
protocol; Table~\ref{tab:crossprotocol} reports the shape this literature uses, and
it should not be read across the rule.

\paragraph{The rate comparison.}
The $2.2704~\bpw$ of our uniform p2 tier and their $2.5425$ share a denominator, exclude the same modules, and
both intend to count payload plus every side channel, so they are the same
convention.  Two things are nonetheless charged unevenly, and both favour us.
Their figure is measured on-disk bytes, and those bytes include a redundant copy of
the rank indices that is bit-identical to the original ($0.0413~\bpw$) together
with container metadata ($0.0156$); ours is the logical rate the format declares,
with no packer written and therefore no container overhead of any kind audited into
it.  Refunding those two fields, the reading most favourable to them, puts their
expert segment at $2.4855~\bpw$.  They also exempt three super-experts from
quantization ($0.0078~\bpw$) where we quantize all $6{,}144$.  Under every one of
these readings our uniform p2 tier stores at least $0.21~\bpw$ less than their
checkpoint, so the comparison below is drawn with the cheaper arm on our side.

\paragraph{Format versus allocation.}
A comparison against a foreign baseline merges two terms, the weight format and the
allocation.  The format term is large: with no allocation at all, the uniform p2 tier
of our ladder beats the released \bitsmoe checkpoint by $2.25$ MMLU, $6.43$ MMLU-Pro
and $10.00$ GSM8K points and $31.4\%$ in perplexity (Table~\ref{tab:storeread}).
Because that arm carries no allocation information, none of this margin is due to bit
allocation.  The whole gain of a read table over uniform p2 is several times
smaller: at $2.6~\bpw$ it moves MMLU by $+0.56$ ($p=0.48$), MMLU-Pro by $+1.21$
($p=0.18$) and GSM8K by $+1.50$, and it reads $0.33~\bpw$ more to do so, so even
this bundles allocation with extra rate.  The rate-matched, same-format controls of
Section 4.2 in the main paper separate the two and are the measure of the allocation.

\paragraph{Scope of the comparison.}
We have not re-run its quantization pipeline, which is not released; the comparison
uses the released checkpoint as published.  The allocation unit, the distortion
surrogate and the calibration set all differ from ours, so the accuracy gap in
Table 1 in the main paper is not evidence about any one of those choices in isolation.
And the read-rate figure above is an accounting quantity on both sides: our own
artifacts are dequantized BF16 tensors with no packed format, so the
$2.95\times$ store-to-read ratio is a property of the two rate models, not a
measured bandwidth.

\section{Measurements on Earlier Instances of the Stream}
\label{app:earlier}

The measurements here were made on two earlier instances of the same
five-precision design, before the instance of Section 3 in the main paper
existed.  Their structural conclusions bear on Sections 3.1 and 3.2 in the main paper; their
magnitudes are not interchangeable with Section 4 in the main paper and are
never mixed with it.  The two-budget, multi-workload and online measurements run on an additive
instance whose audited ladder is $2.5204$, $3.5335$, $4.5322$, $5.5309$,
$6.5295~\bpw$; the distortion ablation runs on the immediate predecessor of the
main instance, whose ladder runs from $2.5204$ to $6.5204~\bpw$ in unit steps.  Both use
activation-weighted distortion with prefill-granularity routing counts from one
calibration corpus, so distortion and rate are weighted by the same frequency.

\subsection{Holding one functional and reading the other}
\label{app:earlier-two}

\begin{table}[h]
\centering
\caption{Constraining one rate and reading the other, on the additive instance.
Store-only overspends the read target on all nine cells; read-only never exceeds
the stored target.}
\label{tab:app-twobudget}
\small
\begin{tabular}{llrrrr}
\toprule
Target & Domain & Store-only: read & overspend & Read-only: store & shortfall \\
\midrule
3.60 & WikiText-2 & 4.718 & $+31.1\%$ & 3.572 & $-0.028$ \\
     & GSM8K      & 4.568 & $+26.9\%$ & 3.522 & $-0.078$ \\
     & code       & 4.411 & $+22.5\%$ & 3.560 & $-0.040$ \\
3.20 & WikiText-2 & 4.092 & $+27.9\%$ & 3.176 & $-0.024$ \\
     & GSM8K      & 3.972 & $+24.1\%$ & 3.151 & $-0.049$ \\
     & code       & 3.874 & $+21.1\%$ & 3.156 & $-0.044$ \\
2.90 & WikiText-2 & 3.5275 & $+21.6\%$ & 2.8865 & $-0.014$ \\
     & GSM8K      & 3.456 & $+19.2\%$ & 2.861 & $-0.039$ \\
     & code       & 3.398 & $+17.2\%$ & 2.845 & $-0.055$ \\
\bottomrule
\end{tabular}
\end{table}

The mechanism behind Table~\ref{tab:app-twobudget} is the shadow price of
Appendix~\ref{app:cancellation}, the storage channel of Proposition 2 in the main paper.  Constrained on storage alone at $3.60~\bpw$ the
assigned precision has rank correlation $0.70$--$0.80$ with routing frequency and
$0.33$--$0.58$ with sensitivity $s_e=d_e(\mathrm{p2})-d_e(\mathrm{p6})$; adding
the frequency-weighted read term reverses this to $0.01$--$0.13$ and
$0.71$--$0.87$.  Over the routed experts ($5{,}735$--$6{,}073$ per domain)
frequency and sensitivity themselves have rank correlation at most $0.06$ in
absolute value, with sensitivity spanning $3.9$--$6.0$ orders of magnitude
against $4.6$--$5.0$ for frequency --- which is why an allocator that follows
routing frequency through the storage term is following a nearly orthogonal signal.  The
asymmetry between the two columns is that under same-domain calibration the read
constraint subsumes the storage constraint; the opposite ordering was seen once
on this project when the reference frequencies came from a different domain than
the calibration, which is a domain mismatch rather than a property of a single
constraint.

\subsection{One artifact for several workloads}
\label{app:earlier-union}

When one artifact serves $D$ workloads, expert $e$ must be stored at
$\max_d M_{d,e}$ while workload $d$ reads $M_{d,e}$; the storage constraint
couples workloads and the read constraints do not.  The per-expert decision
becomes a tuple, solved exactly by enumerating $|\mathcal M|^{D}=125$
combinations at $D=3$.  At a $3.60~\bpw$ read budget the union of three
per-workload solutions stores $3.749~\bpw$ against $6.530$ for retaining every
precision, and constraining the union jointly removes a further $0.188~\bpw$
($5.0\%$) at an unchanged read rate on every workload and a perplexity change of
$+0.003$; pushing to $90\%$ of the union moves the read rates off their targets,
so the allowance that is free on this data ends near $5\%$.  Truncating the stored ceiling instead --- the deployment
heuristic this replaces --- costs $8.5796\to8.6002\to8.7239$ perplexity as the
ceiling drops from p5 to p4 to p3.  The two-stage design of
Section 3.2 in the main paper is the limit of this construction in which one
stored table is fixed and each workload's read table is capped by it.

\subsection{Online recalibration across domains}
\label{app:earlier-online}

\begin{figure}[h]
  \centering
  \includegraphics[width=\linewidth]{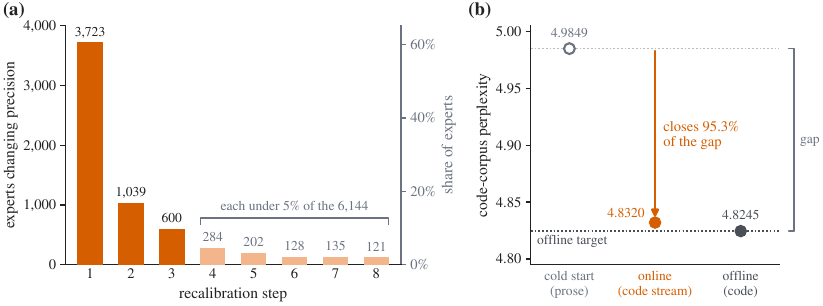}
  \caption{Online recalibration on a $150$-document code corpus from a
  prose-calibrated allocation at $3.2~\bpw$ (earlier additive instance).
  \textbf{(a)}~Experts changing precision at each step; after the third, fewer than
  $5\%$ move.  \textbf{(b)}~Code-corpus perplexity: updating from the stream's own
  unlabelled inputs closes $95.3\%$ of the gap to an allocation calibrated offline on
  code.}
  \label{fig:online-convergence}
\end{figure}

Figure~\ref{fig:online-convergence} runs the online recalibration of
Section 3.2 in the main paper on a 150-document code corpus, starting from prose
statistics; it closes $95.3\%$ of the distance to an allocation calibrated offline
on code ($4.9849\to4.8320$ against $4.8245$ perplexity), with
$3{,}723$, $1{,}039$ and $600$ of $6{,}144$ experts changing precision at the
first three steps and under $5\%$ thereafter.  The size of the first step is the
size of the domain shift: the per-expert distortion weights of the two domains
have Spearman correlation $-0.38$.

\subsection{The distortion functional}
\label{app:earlier-flip}

\begin{table}[h]
\centering
\caption{The distortion functional at four budgets on the predecessor
instance, MMLU accuracy ($n=1{,}425$).  The control is a rate-matched random
assignment between two adjacent tiers.}
\label{tab:app-flip}
\small
\begin{tabular}{llrrr}
\toprule
Budget & Distortion functional & $R_{\mathrm{store}}$ & $R_{\mathrm{read}}$ & MMLU \\
\midrule
2.600 & none (rate-matched control) & 2.600 & 2.595 & 81.26 \\
      & activation-weighted squared error & 2.600 & 2.598 & 80.98 \\
      & decision-aligned            & 2.600 & 2.600 & 81.19 \\
\addlinespace
2.900 & none                        & 2.900 & 2.881 & 81.54 \\
      & squared error               & 2.900 & 2.912 & 80.42 \\
      & decision-aligned            & 2.845 & 2.900 & 81.47 \\
\addlinespace
3.200 & none                        & 3.200 & 3.189 & 82.32 \\
      & squared error               & 3.191 & 3.220 & 81.33 \\
      & decision-aligned            & 3.098 & 3.200 & 82.25 \\
\addlinespace
3.600 & none                        & 3.600 & 3.595 & 82.95 \\
      & squared error               & 3.569 & 3.609 & 82.60 \\
      & decision-aligned            & 3.439 & 3.600 & 82.67 \\
\bottomrule
\end{tabular}
\end{table}

Table~\ref{tab:app-flip} varies the output weight $A$ and nothing else.  The
squared-error criterion sits below the uninformed control at all four budgets;
the decision-aligned criterion runs within $0.07$ points of the control at three
budgets while storing $0.06$--$0.16~\bpw$ less at the three upper budgets.  At
$3.200~\bpw$, where per-question logs support a paired test, it is $0.07$ points
from the control ($p=1.00$) and $0.91$ above the squared-error arm ($p=0.16$).
On the generative benchmarks the decision-aligned arm gains $+0.50$, $+2.25$ and
$\pm0.00$ GSM8K points (flexible-extract) and $+3.40$, $+2.60$ and $-1.20$ MATH-500
points (Math-Verify) over the control at $2.9$, $3.2$ and $3.6~\bpw$.  No single cell
is significant; pooled over both benchmarks, the two lower budgets give $+2.28$ points
($95\%$ CI $[+0.50,+4.11]$, $p=0.013$) and the $3.6~\bpw$ budget gives $-0.67$
($p=0.56$).  This ablation predates the controls of Table~\ref{tab:where}, so its positive cells
are indicative rather than a measurement of what the allocation is worth.

\section{Evidence from an Earlier Tier Family, and Negative Results}
\label{app:hif}

\subsection{A tier family that was not on the convex hull}

Before the prefix-nested stream existed, this project allocated over a family of
fixed-precision formats occupying $2.6$, $3.6$, $4.6$ and $8.0~\bpw$ and stored
separately, so a mixed allocation needed one copy per tier and the two rates were
held in a single Lagrangian.  The magnitudes measured there do not transfer to
the stream of Section 3 in the main paper, whose tiers are nested prefixes of
one code where these were separately stored formats, but three controls run there have
not been repeated, and their signs bear on
Section 3.1 in the main paper.  All rates in this subsection are realized read rates under
the routing of the scored benchmark; a larger MMLU sample ($70$ questions per
subject, BF16 $83.18$) is used where noted.

\paragraph{A frequency heuristic against a sensitivity allocator.}
The heuristic in use before the allocator assigned the higher tier to the $30\%$
most frequently routed experts.  At a realized read rate of $3.68~\bpw$ it scored
$78.87$ on the larger MMLU sample against $83.01$ for the sensitivity allocator
at $3.60$ and $82.03$ for a uniform $3.6~\bpw$ format; on MMLU-Pro $50.43$
against $56.36$ and $55.29$; on WikiText-2 $11.06$ against $8.50$ and $9.01$.
Assigning by routing frequency was worse than not allocating at all, which is what the shadow
price of Appendix~\ref{app:cancellation} predicts when frequency and sensitivity
are uncorrelated.  On GSM8K at a realized $3.6~\bpw$ all three tied ($89.75$,
$88.75$, $89.00$), the first observation on this project that allocation has no
lever where the rate--accuracy curve is flat.

\paragraph{A rate-structure-only control.}
To separate what the distortion table contributes from what the rate constraint
alone imposes, a control was built that reproduces the allocator's per-layer tier
histogram exactly and its read rate to within $0.03\%$, but assigns tiers by
frequency rank alone with no distortion information.  Against it the sensitivity
allocator gained $6.10\%$ WikiText-2 perplexity at $3.2~\bpw$ ($62$ of $62$
documents, $p=4\times10^{-19}$) and $2.76\%$ at $3.6$ ($59$ of $62$,
$p=2\times10^{-14}$), and $2.21$ ($p=0.019$) and $1.71$ ($p=0.042$) MMLU-Pro
points.  The control itself scored below the uniform $3.6~\bpw$ format on
MMLU-Pro while beating it on both WikiText-2 and code perplexity, so mixed precision without a
distortion model is not a free improvement even at a fixed histogram.

\paragraph{The distortion functional at one realized rate.}
Five allocations held to a realized MMLU read rate of $3.000~\bpw$ (Figure~\ref{fig:functional}), differing
only in the distortion functional, gave WikiText-2 perplexities of $8.99$
(activation-weighted squared error), $9.19$ (cross-entropy-gradient weighted),
$9.28$ (decision-aligned), $9.46$ (unweighted squared error) and $10.57$ (final
position only), strictly ordered by how narrowly the functional concentrates on
the decision; and MMLU-Pro accuracies of $53.71$, $55.07$, $55.71$, $55.00$ and
$55.00$, nearly the mirror image, with the decision-aligned arm the only one
significant against the squared-error arm ($+2.00$, McNemar $p=0.030$;
$p=0.12$ after correcting for four arms).  When the decision-aligned weight was
calibrated on the served domain it scored $86.00$ GSM8K under flexible extraction
($85.50$ strict) against $81.00$ for squared error ($p=0.015$); taken from a
different domain it did not ($79.50$ strict against $81.25$, $p=0.50$), which is
the observation that made same-domain calibration of $A$ a requirement.

\begin{figure}[h]
\centering
\includegraphics[width=\linewidth]{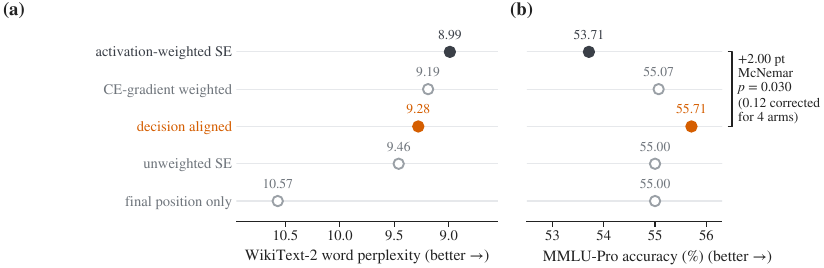}
\caption{Five distortion functionals at a realized MMLU read rate of $3.000~\bpw$
(earlier tier family).  \textbf{(a)}~WikiText-2 perplexity; \textbf{(b)}~MMLU-Pro
accuracy, nearly its mirror image.  The decision-aligned functional used by \duorate
is the only arm above activation-weighted squared error on MMLU-Pro at a nominal
$p<0.05$.}
\label{fig:functional}
\end{figure}

\paragraph{The two reference formats.}
Two members of that family are the fixed-precision references this project carried
before the uniform tiers of the prefix stream existed.  The $4.6~\bpw$ member is the
HIF4 row of Table 1 in the main paper (listed there at the $4.5~\bpw$ of the published format).  The $2.6~\bpw$ member scores $47.37$ WikiText-2
perplexity, $37.21$ MMLU-Pro and $68.25$ GSM8K under our harness; its collapse is a
property of a format without a refinement ladder and is not a control for
allocation, which is why our same-format controls are the uniform tiers of the
prefix stream instead.

\subsection{A published allocator's objective on the current stream}
\label{app:controls}

\begin{figure}[h]
\centering
\includegraphics[width=\linewidth]{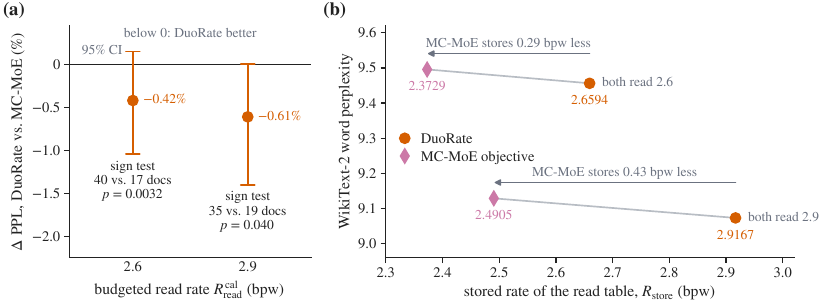}
\caption{The MC-MoE objective inside our bisection, with the same distortion table and
cap (WikiText-2).  \textbf{(a)}~At matched read rate, \duorate's perplexity relative to
MC-MoE's: the sign test detects the difference, the $95\%$ interval includes zero.
\textbf{(b)}~The same tables at their stored rates: MC-MoE stores $0.29$ and
$0.43~\bpw$ less.  Neither dominates.}
\label{fig:mcmoe}
\end{figure}

MC-MoE~\citep{mcmoe} weights the distortion by routing frequency and leaves the bit
term unweighted.  Substituting that objective into the same bisection, with the same
distortion table and the same stage-one cap, loses to ours on WikiText-2 at both
read budgets by margins the sign test sees and the interval does not
(Figure~\ref{fig:mcmoe}): $0.42\%$ at
$2.6~\bpw$ ($40$ documents against $17$, $p=0.0032$, interval $[-1.04,+0.15]$) and
$0.61\%$ at $2.9$ ($35$ against $19$, $p=0.040$, interval $[-1.40,+0.01]$).
Weighting the distortion by frequency drives the experts a workload rarely reaches
to the floor; deepening them costs nothing in the read functional and a full tier in
the unweighted one, so those tables store $2.3729$ and $2.4905~\bpw$ against our
$2.6594$ and $2.9167$ at the same read rates, cheaper by $0.29$ and $0.43~\bpw$.
Matched on the read rate, ours scores better and costs more to hold, so neither
dominates, and which one looks better is decided by which of the two functionals a
deployment holds fixed.  That is this paper's claim, demonstrated with a
competitor's objective rather than our own.

This does not contradict Proposition 1 in the main paper.  The proposition
concerns the \emph{surrogate} decode objective, in which dropping $f$ is optimal;
weighting the distortion alone is the shadow-price case of
Appendix~\ref{app:cancellation}, where $f$ survives and the assignment follows
routing frequency.  The exponents $\alpha,\beta,\gamma$ of MC-MoE's objective are
neither given numerically nor ablated in its own paper; we set the gate-weight
exponent to zero, because no routing dump exists for the prose domain, and the
other two to one.

\subsection{Negative results}
\label{app:negative}

\paragraph{Per-projection units.}
Splitting the expert into gate, up and down units improved the proxy objective by
$17\%$ and MMLU by $0.15$ points on the earlier tier family, inside the noise
floor.  It has not been re-measured on the current stream.

\paragraph{Proxy differences.}
Seven times on this project a difference in the proxy objective --- $17\%$, a
tier-histogram shift, $0.1~\bpw$ of ledger --- has failed to appear end to end.
No proxy figure is presented as a result in this paper.

\paragraph{Cross-layer budget.}
Earlier analyses on this project attributed most of the allocation gain to the
distribution of bits across layers rather than within them.  A preliminary probe
at matched stored rate --- two tiers, layer granularity, choosing which $24$ of
$48$ layers receive the higher tier --- scored $76.3$ MMLU when the layers were
chosen by our layer-level distortion signal, $78.1$ when chosen at random and
$79.6$ when the shallowest $24$ were chosen.  The runs are single and unpaired,
they use the earlier tier family, and the working point differs from the tables
above, so no conclusion is drawn in either direction and no claim about the
cross-layer budget is made in this paper.

\end{document}